\documentclass{article}

\usepackage[accepted]{icml2026}

\usepackage[utf8]{inputenc}   
\usepackage{amsmath}
\usepackage{amsthm}
\usepackage{amssymb}
\usepackage{amsfonts}
\usepackage{bm}
\usepackage{dsfont}
\usepackage{mathtools}
\usepackage{pifont}
\usepackage{xspace}
\usepackage{enumitem}
\usepackage{graphicx}
\usepackage{booktabs}

\usepackage{siunitx}
\usepackage[svgnames]{xcolor}
\usepackage{transparent}
\DeclareRobustCommand{\lightred}[1]{{\transparent{0.4}\textcolor{red}{#1}}}
\DeclareRobustCommand{\lightblue}[1]{{\transparent{0.4}\textcolor{blue}{#1}}}
\definecolor{indigo}{RGB}{75, 0, 130}
\definecolor{mediumpurple}{RGB}{147, 112, 219}
\definecolor{darkred}{RGB}{255, 0, 0}
\definecolor{darkyellow}{RGB}{192, 192, 0}

\usepackage{hyperref}		 
\hypersetup{
    colorlinks      = true,     
    urlcolor        = blue,     
    linkcolor       = purple,   
    citecolor       = violet    
}

\newcommand{\dw}{\Delta\vw}
\newcommand{\g}[1]{\bm{g_{#1}}}
\newcommand{\gr}{\bm{g_r}}
\newcommand{\gf}{\bm{g_f}}
\newcommand{\grb}{\bar{\bm{g}}_{\bm{r}}}
\newcommand{\gfb}{\bar{\bm{g}}_{\bm{f}}}
\newcommand{\gfu}{\hat{\bm{g}}_{\bm{f}}}
\newcommand{\grbp}{\mathop{\grb}\nolimits_{\perp}}
\newcommand{\gfbp}{\mathop{\gfb}\nolimits_{\perp}}
\newcommand{\Hr}{\bm{H_r}}
\newcommand{\eigenmin}{\lambda_{\text{min}}}
\newcommand{\eigenmax}{\lambda_{\text{max}}}
\newcommand{\perpen}[1]{{#1}_{\perp}}

\newcommand{\ours}{HAMU\xspace}
\newcommand{\primal}{HAMU-Q\xspace}
\newcommand{\reciprocal}{HAMU-U\xspace}

\renewcommand{\phi}{\varphi}
\renewcommand{\epsilon}{\varepsilon}

\newcommand{\lp}{\left(}
\newcommand{\rp}{\right)}
\newcommand{\lb}{\left[}
\newcommand{\rb}{\right]}
\newcommand{\lc}{\left\{}
\newcommand{\rc}{\right\}}

\newcommand{\floor}[1]{\left\lfloor #1 \right\rfloor}

\newcommand{\norm}[1]{\left\lVert #1 \right\rVert}
\newcommand{\abs}[1]{\left\lvert #1 \right\rvert}

\newcommand{\ud}{\mathrm{d}} 
\newcommand{\pd}[2]{\displaystyle\frac{\partial #1}{\partial #2}}

\newcommand{\squishstart}{
  \begin{list}{$\bullet$}{
    \setlength{\itemsep}{1pt}
    \setlength{\parsep}{0pt}
    \setlength{\topsep}{1pt}
    \setlength{\partopsep}{0pt}
    \setlength{\leftmargin}{1em}
    \setlength{\labelwidth}{1.5em}
    \setlength{\labelsep}{0.5em} 
  } 
}
\newcommand{\squishend}{
  \end{list}
}

\DeclarePairedDelimiterX{\infdivx}[2]{(}{)}{%
  #1\;\delimsize|\delimsize|\;#2%
}

\def\vzero{{\bm{0}}}

\def\vp{{\bm{p}}}

\def\vu{{\bm{u}}}
\def\vv{{\bm{v}}}
\def\vw{{\bm{w}}}
\def\vx{{\bm{x}}}

\def\vz{{\bm{z}}}

\def\mA{{\bm{A}}}
\def\mB{{\bm{B}}}

\def\mH{{\bm{H}}}
\def\mI{{\bm{I}}}

\newcommand{\bR}{\mathbb{R}}

\newcommand{\cL}{\mathcal{L}}

\theoremstyle{plain}
\newtheorem{theorem}{Theorem}[section]
\newtheorem{proposition}[theorem]{Proposition}
\newtheorem{lemma}[theorem]{Lemma}
\newtheorem{corollary}[theorem]{Corollary}
\theoremstyle{definition}

\theoremstyle{remark}

\icmltitlerunning{How Hard Can It Be? Hardness-Aware Multi-Objective Unlearning}

\begin{document}

\twocolumn[
  \icmltitle{How Hard Can It Be? Hardness-Aware Multi-Objective Unlearning}



  \icmlsetsymbol{equal}{*}

  \begin{icmlauthorlist}
    \icmlauthor{Jiangwei Chen}{equal,nus,i2r}
    \icmlauthor{Xinyuan Niu}{equal,nus,cfar}
    \icmlauthor{Rachael Hwee Ling Sim}{nus}
    \icmlauthor{Zhengyuan Liu}{i2r}
    \icmlauthor{Nancy F. Chen}{i2r,cfar}
    \icmlauthor{Bryan Kian Hsiang Low}{nus}
  \end{icmlauthorlist}

  \icmlaffiliation{nus}{Department of Computer Science, National University of Singapore, Singapore}
  \icmlaffiliation{i2r}{Institute for Infocomm Research, Agency for Science, Technology and Research, Singapore}
  \icmlaffiliation{cfar}{Centre for Frontier AI Research, Agency for Science, Technology and Research, Singapore}

  \icmlcorrespondingauthor{Bryan Kian Hsiang Low}{lowkh@comp.nus.edu.sg}

  \icmlkeywords{Machine Learning, ICML}

  \vskip 0.3in
]



\printAffiliationsAndNotice{\icmlEqualContribution}

\begin{abstract}
Machine unlearning aims to remove the influence of specific forget training data due to privacy, copyright or bias concerns while maintaining the model performance on the remaining retain data.
Existing unlearning algorithms, such as optimizing a weighted combination of losses, have tried to achieve these objectives of improving forget quality and maintaining retain utility.
However, they do not guarantee that these objectives can be improved by a specified extent for all forget and retain data.
In this work, we address this limitation with a novel and theoretically-grounded approach from a constrained optimization perspective.
Firstly, we identify that the \emph{hardness} of reconciling both objectives can be quantified by the similarity between the forget data and the retain data.
Next, we derive an unlearning algorithm~(\ours) with the overall goal of guaranteeing a specified improvement in forget quality while minimizing the retain utility cost/degradation by updating the model weights based on our hardness measure. 
Our hardness measure also informs users when retain utility degradation is unavoidable,
i.e.,~both objectives cannot be improved simultaneously, and stopping should be considered.
Our algorithm is applicable to non-convex models and is easily parallelizable, making it readily deployable in real-world scenarios.
We empirically demonstrate \ours's superior performance over baselines on both image and text datasets using large models.
Our code is available at \url{https://github.com/aoi3142/HAMU}.
\end{abstract}

\section{Introduction}
Machine unlearning~\citep{Cao2015-uz} presents an appealing paradigm to combat the data issues arising from the prevalent adoption of large \emph{machine learning}~(ML) models~\citep{Raza2025, 10847310},
such as privacy leakage~\citep{pmlr-v235-pawelczyk24a}, copyright infringement~\citep{eldan2023s} and harmful bias~\citep{DBLP:conf/aaai/WuZYWCZHZ025}.
Machine unlearning aims to fulfill both {objectives}:
\textbf{(1)} improve \emph{forget quality} by removing the influence of some forget data, from the ML models;
while \textbf{(2)} maintaining \emph{retain utility}, the ML model performance on the remaining retain data.
Unlearning usually achieves these two objectives by increasing the loss on the forget data~\citep{Zhang2024-yj} or approximating the retrained model which minimizes the loss on the retain data~\citep{Neel2021-bv}.

Most existing unlearning algorithms have neglected the possible conflict between both objectives.
To achieve \textbf{(1)}, unlearning may incur some \emph{cost of forgetting}, which is the negative impact on the ML model performance~(retain utility).
For example, it may lead to catastrophic forgetting where the unlearned model performs considerably worse than the retrained model~\citep{Nguyen2020-an, fan2025simplicity} or more generally retain utility degradation where the ML model performance or knowledge of the remaining retain data worsens, sacrificing \textbf{(2)}.
We term this phenomenon of improving forget quality causing retain utility degradation as \emph{collateral forgetting}.
Existing works have considered multi-objective unlearning~\citep{jin-etal-2025-unlearning} that attempts to simultaneously improve both objectives, however, it may not be feasible to do so for all possible forget and retain data.
An extreme and intuitive example is that collateral forgetting is unavoidable when the forget and retain data are identical.
This conflict oversight raises the following fundamental questions:
\textbf{How much is the \emph{cost of forgetting}~(in terms of retain utility degradation) to improve forget quality by a specified extent? 
How hard is it to reconcile both objectives in machine unlearning?}

We answer the above questions by quantifying the hardness of unlearning, $\kappa$, as a model-dependent continuous measure that depends on the similarity between the forget and retain data.
Then, we establish a theoretical connection which shows that the cost of forgetting, in terms of retain utility degradation, generally increases with the hardness measure~(Fig.~\ref{fig:hardness_regimes}).
For example, the cost is the highest in the hardest scenario where the forget and retain data are identical.
Moreover, we identify a hardness threshold $\kappa_2$~(Fig.~\ref{fig:hardness_regimes}), whose exceedance informs the model owner prior to unlearning that 
local collateral forgetting is unavoidable and small updates to improve forget quality will degrade retain utility~(i.e.,~the cost of forgetting is positive).

\textbf{Next, how can the model owner ensure that the forget quality improves by a specified extent and reduces the cost of forgetting, even for non-convex\footnote{Methods with theoretical guarantees usually require strong assumptions such as strongly-convex loss functions~\citep{Ullah2021-df, Allouah2025-en}, which limit their applicability to non-convex models.} ML models such as \emph{neural networks}~(NNs)?} 
At first glance, one solution would be to use a carefully tuned weighted combination of the forget and retain losses. 
However, this simple solution does not allow the model owner to guarantee that the forget quality would increase by a specified extent~\citep{Kurmanji2023-xm, duan2025ready2unlearn}.
In contrast, we formulate unlearning in each iteration as a constrained optimization problem where the goal is to minimize the cost of forgetting~(the first-order approximation of the retain utility degradation) while constraining that the linear approximation of the forget quality must improve by the specified extent. 
Our hardness-aware multi-objective unlearning~(\ours) algorithm decides on the unlearning weight update based on the hardness measure~(Fig.~\ref{fig:hardness_regimes}).
For harder scenarios, our \ours algorithm uses the rectified update which incorporates the forget data's gradients to guarantee forget quality improvement.
\ours is efficient but applicable to non-convex models as we consider the first-order approximation of the loss function limited to a small local radius.
The linear and layer-wise independence nature of \ours at each iteration also makes it scalable and parallelizable for use with \emph{large language models}~(LLMs).
We empirically show that \ours outperforms other baselines in terms of unlearning performance and reduces the cost of forgetting through extensive experiments.

To summarize, our key contributions include:
\squishstart
\item We introduce a new unlearning perspective --- the goal is to improve an objective (e.g.,~forget quality) by a specified extent while optimizing the other objective (e.g.,~minimizing the cost of forgetting/retain utility degradation)~(Sec.~\ref{sec:theory}).
\item We quantify the hardness of unlearning and establish a \emph{theoretical connection} between hardness and cost of forgetting~(Sec.~\ref{sec:hardness}). The hardness measure also determines the unlearning weight update in our \ours algorithm~(Sec.~\ref{sec:update}) and when \emph{collateral forgetting} occurs~(Sec.~\ref{sec:collateral}).
\item We adapt \ours to efficiently and effectively work for large, non-convex ML models~(Sec.~\ref{sec:algo}).
\item We demonstrate \ours consistently outperforms baselines across datasets and model architectures through extensive experiments~(Sec.~\ref{sec:exp}).
\squishend
\begin{figure}[t]
    \centering
    \includegraphics[width=\linewidth]{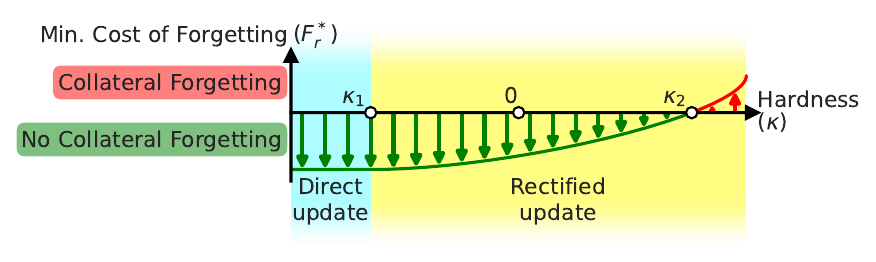}
    \caption{Hardness regimes characterized by the measure $\kappa$ for fixed $\norm{\gr}$ and $\norm{\gf}$.
    Higher $\kappa$ corresponds to higher cost of forgetting in terms of retain utility degradation.
    The threshold $\kappa_1$ and $\kappa_2$, respectively, determines when the rectified update and unavoidable retain utility degradation (i.e.,~collateral forgetting) occurs to improve the forget quality by at least $Q$.}
    \label{fig:hardness_regimes}
\end{figure}

\section{Multi-Objective Machine Unlearning}
We consider a supervised learning task on the dataset $D \triangleq \{(\vx_i, y_i)\}_{i=1}^n$, where $y_i \in [K] \triangleq \{1, \ldots, K\}$ represents classes or vocabulary tokens.
A model~(e.g.,~LLM) parameterized by $\vw \in \bR^d$ is trained or evaluated using the cross-entropy loss $L(\vp(\vx; \vw), y) \triangleq -\log p_y(\vx; \vw)$ for each sample $(\vx, y)$,
where $\vp(\vx; \vw) \in \Delta^{K-1}$ denotes the predicted class probabilities~(for classification) or distribution over the next token~(for next token prediction).
In machine unlearning, the training data $D$ is partitioned into the forget data $D_f \subseteq D$ and the retain data $D_r \triangleq D \setminus D_f$.
Let $\vw_0$ denote the weights of the original model that is trained to minimize the averaged cross-entropy loss on $D$.
At iteration $t$~(starting with $t=0$), the unlearning algorithm updates the model weights $\vw_t$ using information from $D_f$ or $D_r$, or combination thereof, depending on the objectives, to obtain $\vw_{t+1}$.

One common objective in machine unlearning is to approximate the retrained model~\citep{Georgiev2025-aj, DBLP:conf/emnlp/ChenY23},
i.e.,~the model trained from scratch on $D_r$.
This is usually achieved by minimizing the retain loss via fine-tuning~(e.g.,~gradient descent on $D_r$) the original model on the retain data~\citep{Neel2021-bv, bui2024newton}, which helps to maintain retain utility.
Due to the inadequacy of this objective to ensure forget quality~\citep{block2025machine}, the objective of increasing forget loss~(e.g.,~gradient ascent on $D_f$) has been proposed~\citep{Halimi2022-aj, Zhang2024-yj}.
However, these gradient-ascent methods may induce retain utility degradation~\citep{mavrothalassitis2025ascent}, making increasing forget loss insufficient as the sole objective~\citep{fan2025simplicity}.

Hence, instead of solely approximating the retrained model or increasing the forget loss, we focus on the multi-objective aspect of unlearning.
Existing multi-objective unlearning methods typically optimize a weighted combination of the objectives~\citep{Kurmanji2023-xm, duan2025ready2unlearn}.
However, there is no guarantee that the desired performance can be achieved for any of the objectives.
Furthermore, they have not considered the inherent statistical dependence between the forget and retain data, hence there is no attempt to reconcile the possibly conflicting objectives.
The two objectives of ensuring forget quality and maintaining retain utility are not the same for non-convex models.
In fact, we show that depending on the relationship between $D_f$ and $D_r$, these two objectives can be either conflicting or compatible.

\section{Conflict Analysis via Learning Dynamics}
\label{sec:conflict}
We begin by analyzing how taking a single gradient descent step on a sample $\vz' = (\vx', y')$ during unlearning affects the loss on another sample $\vz = (\vx,y)$.
Let $L_t(\vz) \triangleq L(\vp(\vx; \vw_t), y)$ denote the loss on $\vx$ at iteration $t$.
We are interested in the quantity $\Delta_{\vz'} L_t(\vz) \triangleq L_{t+1}(\vz) - L_t(\vz)$ for different choices of $\vz'$ and $\vz$,
where $\Delta_{\vz'}$ means the change in loss is induced by taking a single step on the sample $\vz'$.
We first consider the simplified case of fine-tuning on a single sample from the retain data.
Using first-order approximation\footnote{We defer the detailed derivation and discussion on why first-order approximation is sufficient to App.~\ref{app:first-order-loss-approximation}.}, for any $\vz_r = (\vx_r, y_r) \in D_r$, 
\begin{equation*}
    \Delta_{\vz_r} L_t(\vz_r) = L_{t+1}(\vz_r) - L_t(\vz_r) \approx -\eta \cdot \norm{\nabla_{\vw} L_t(\vz_r)}^2,
\end{equation*}
where $\eta$ is the learning rate, and $\nabla_{\vw}$ denotes the gradient with respect to the model parameters $\vw$. 
Hence, unlearning by fine-tuning on a sample from the retain data reduces the loss on that sample.
Now consider how fine-tuning on the same sample impacts the loss on a sample from the forget data.
For any $\vz_f = (\vx_f, y_f) \in D_f$ we have
\begin{equation*}
    \Delta_{\vz_r} L_t(\vz_f) \approx -\eta \cdot \nabla_{\vw} L_t(\vz_f) \cdot \nabla_{\vw} L_t(\vz_r).
\end{equation*}
Unlike the previous scenario, $\Delta_{\vz_r} L_t(\vz_f)$ is not always negative, but depends on the alignment between the two gradients $\nabla_{\vw} L_t(\vz_f)$ and $\nabla_{\vw} L_t(\vz_r)$.
When the two \textbf{gradients are aligned}~(positive dot product), the losses on the forget and retain samples both decrease.
When the \textbf{gradients are misaligned}~(negative dot product), the loss on the forget sample increases as desired.
The analysis naturally extends to batch gradients $L_t(D) \triangleq 1/|D| \sum_{\vz \in D} L_t(\vz)$.
When taking a single gradient step $-\eta \nabla_{\vw} L_t(D_r)$ to improve retain utility, the change in the batch forget loss and quality depends on the dot product $\nabla_{\vw} L_t(D_f) \cdot \nabla_{\vw} L_t(D_r)$.
Negative and positive gradient dot product leads to better and worse forget quality, respectively.
Intuitively, the dot product fundamentally determines how hard it is to reconcile both objectives of improving retain utility and forget quality in unlearning.
In the next section, we will show that \textbf{the gradient dot product also monotonically determines the minimum cost of forgetting, i.e.,~cost in retain utility to improve forget quality by a specified extent.}

\section{First-Order Convex Formulation}
\label{sec:theory}
Since optimizing for smaller $L_t(D_r)$ and larger $L_t(D_f)$ simultaneously can be conflicting and provides no guarantee for satisfying either objective, we consider the problem of optimizing for the best retain utility while stipulating a minimum requirement on the forget quality,
i.e.,~minimize $L_t(D_r)$ subject to $\Delta L_t(D_f) > Q$,
where $Q > 0$ is the per-iteration threshold for the forget quality requirement.
Following the first-order approximation in Sec.~\ref{sec:conflict}, we focus on the objectives in each iteration.
Specifically, at iteration $t$, the retain utility degradation (i.e.,~cost of forgetting) $\Delta L_t(D_r) \approx \gr \cdot \dw$ and the improvement in forget quality $\Delta L_t(D_f) \approx \gf \cdot \dw$, 
where $\gf \triangleq \nabla_{\vw} L_t(D_f)$, $\gr \triangleq \nabla_{\vw} L_t(D_r)$ are the averaged gradients,
and $\dw \triangleq \vw_{t+1} - \vw_t$ is the weight update at iteration $t$.
To ensure accurate approximation~(see App.~\ref{app:first-order-loss-approximation}), we restrict the per-iteration update to a local neighborhood,
i.e.,~$\norm{\dw} \leq R$ for some small $R > 0$.
Therefore, it is equivalent to solve the following optimization problem\footnote{Unlike multi-objective optimization problems where all objectives are \emph{minimized}, our formulation only minimizes the retain utility degradation and improves the forget quality~(i.e.,~\emph{minimizing} retain loss while \emph{increasing} forget loss) during unlearning.}:
\begin{equation} 
\label{problem:opt-ldr}
\begin{aligned}
    & \underset{\dw}{\text{minimize}}  && \gr \cdot \dw \\
    & \text{subject to}                && \gf \cdot \dw \geq Q, \quad \norm{\dw} \leq R .
\end{aligned}
\end{equation}
Our formulation is general and allows the use of alternative loss functions such as the NPO loss~\citep{Zhang2024-yj}.
We state our main results below and defer the full derivations to App.~\ref{app:proof}.

\subsection{Hardness of Unlearning}
\label{sec:hardness}
Problem~\eqref{problem:opt-ldr} is feasible when $Q \leq R \norm{\gf}$, under this feasible condition we can solve the problem analytically.
The following theorem provides the closed-form expression for the optimal cost of forgetting in terms of retain utility degradation given a minimum forget quality requirement.
\begin{theorem}
\label{theorem:opt-primal}
    Let $\kappa \triangleq \gr \cdot \gf$ be the gradient dot product 
    and $\perpen{\gr} \triangleq \gr - \frac{\gr \cdot \gf}{\norm{\gf}^2} \cdot \gf$ be the component of $\gr$ orthogonal to $\gf$.
    Assume\footnote{This assumption is typically satisfied in practice~(Sec.~\ref{sec:exp}).} $Q \leq R \norm{\gf}$, the optimal value to Problem~\eqref{problem:opt-ldr} and the minimum cost of forgetting $F_r^* \triangleq \min \gr \cdot \dw$ is given by
\begin{equation} 
\label{eq:opt-ldr}
    \begin{dcases}
        - R \norm{\gr} & \text{if } \kappa \leq \kappa_1 \\
        \frac{\kappa Q}{\norm{\gf}^2} - \norm{\perpen{\gr}} \sqrt{R^2 - \frac{Q^2}{\norm{\gf}^2}} & \text{otherwise},
    \end{dcases}
\end{equation}
where $\kappa_1 \triangleq - Q\norm{\gr} / R$.
\end{theorem}
As prior hardness measures are largely heuristic and unable to predict unlearning performance for a given instance~\citep{Zhao2024-wq}, Theorem~\ref{theorem:opt-primal} provides \emph{the first theoretical connection between a measurable quantity $\kappa$ and unlearning objective performance}, quantified by the minimum retain utility degradation subject to a sufficient forget quality improvement.
Since $\norm{\perpen{\gr}}$ also depends on $\kappa$, the precise relationship between $F_r^*$ and $\kappa$ is not immediately apparent.
We formalize this relationship below.
\begin{corollary}[Hardness Measure]
\label{corollary:hardness}
    The minimum cost of forgetting $F_r^*$ is monotonically non-decreasing with $\kappa$ for fixed $\norm{\gr}$ and $\norm{\gf}$.
\end{corollary}
Consequently, $\kappa$ naturally serves as a hardness measure for unlearning.
This makes intuitive sense as a larger $\kappa$ implies $\gr$ and $\gf$ are more aligned, and $D_r$ and $D_f$ are more similar.
Hence, it is more difficult to reconcile both objectives,
i.e.,~minimizing retain loss while ensuring forget quality, 
resulting in a higher cost of forgetting and retain utility degradation $F_r^*$.
Consider the extreme example when $D_r = D_f$, it becomes impossible to maintain retain utility while improving forget quality, as the losses on the forget and retain data are identical, making it a particularly hard instance.
Note that the hardness measure $\kappa$ depends not only on the similarity between $D_r$ and $D_f$, but also on the model itself.
This is intuitive, as reconciling the objectives can become more difficult after unlearning the same forget data $D_f$ for several iterations.

\subsection{Hardness-Aware Unlearning Update}
\label{sec:update}
Interestingly, the hardness measure $\kappa$ not only indicates the difficulty of the unlearning instance and predicts the cost of forgetting, but also informs the optimal unlearning update, as detailed in the following proposition.
\begin{proposition}[Optimal Unlearning Update]
\label{proposition:update-ldr}
    Assume $Q \leq R \norm{\gf}$, the minimum cost of forgetting $F_r^*$ is achieved via the optimal unlearning update $\dw^*$, given by
\begin{equation} 
\label{eq:update-ldr}
    \begin{dcases}
        - \frac{R}{\norm{\gr}} \cdot \gr & \text{if } \kappa \leq \kappa_1 \\
        \frac{Q}{\norm{\gf}^2} \cdot \gf - \sqrt{R^2 - \frac{Q^2}{\norm{\gf}^2}} \cdot \frac{\perpen{\gr}}{\norm{\perpen{\gr}}} & \text{otherwise}.
    \end{dcases}
\end{equation}
\end{proposition}
The threshold $\kappa_1$ determines the level of hardness and governs which type of unlearning update should be applied.
When $\kappa \leq \kappa_1$, the objectives are compatible and unlearning is relatively ``easy''.
It is sufficient to apply the \emph{direct} update in the same direction as $-\gr$ to improve retain utility without violating the constraint on forget quality since $\gf \cdot \dw^* \geq Q$ for $\kappa \leq \kappa_1$.
Conversely, the objectives become conflicting when $\kappa > \kappa_1$.
If we still apply the direct update, the forget quality constraint would be violated as 
\begin{equation*}
    \gf \cdot \lp -\frac{R}{\norm{\gr}} \cdot \gr \rp = -\frac{\kappa R}{\norm{\gr}} < -\frac{R}{\norm{\gr}} \cdot \frac{-Q \norm{\gr}}{R} = Q.
\end{equation*}
Thus, when it is ``hard'' to reconcile the two unlearning objectives, we can no longer move freely in the direction of $-\gr$, but must instead apply the \emph{rectified} update, which incorporates the direction of $\gf$ to ensure the forget quality improves by the specified extent $Q$.
In Fig.~\ref{fig:unlearn_update}, we illustrate our unlearning updates under different scenarios.
\begin{figure}[t]
    \centering
    \includegraphics[width=\linewidth]{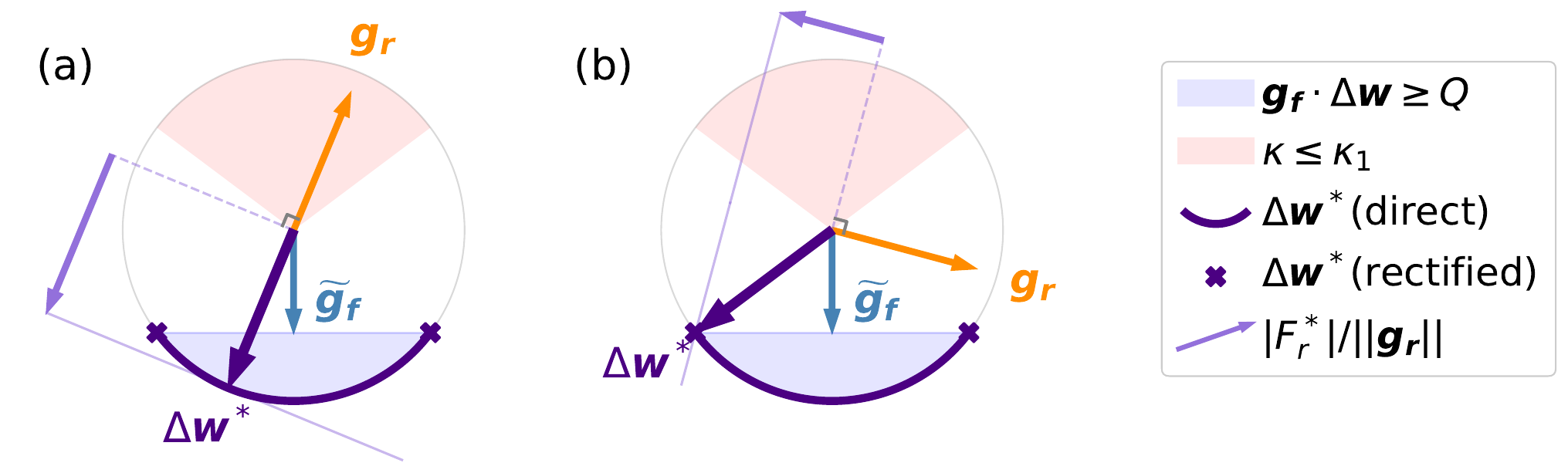}
    \caption{A 2D illustration of \ours update. 
    Problem~\eqref{problem:opt-ldr} restricts the weight update $\dw$ within a circle of radius $R$ and the \lightblue{region} where the forget quality constraint is satisfied. 
    In (a), the retain gradient $\gr$ lies in the ``easy'' \lightred{region} ($\kappa \leq \kappa_1$), hence the \textbf{direct update} is applied, which is in the opposite direction of $\gr$. 
    In (b), the retain gradient lies outside the ``easy'' region ($\kappa > \kappa_1$), the \textbf{rectified update} is applied,
    which mixes $\tilde{\bm{g}}_{\bm{f}}$ and a scaled $\perpen{\gr}$.
    Here, $\tilde{\bm{g}}_{\bm{f}} \triangleq (Q / \norm{\gf}^2) \gf$ and we show all optimal unlearning \textcolor{indigo}{updates} for a fixed $\gf$ and varying $\gr$.
    The resulting change in the retain \textcolor{mediumpurple}{objective} $F_r^* = \gr \cdot \dw^*$ is also plotted.
    Note that in both (a) and (b), there is no collateral forgetting ($F_r^* < 0$), and a larger $\abs{F_r^*}$ indicates better resulting retain utility.}
    \label{fig:unlearn_update}
\end{figure}

\subsection{Unavoidable Collateral Forgetting}
\label{sec:collateral}
Problem~\eqref{problem:opt-ldr} only imposes a constraint on the forget quality, and optimizes for the lowest cost of forgetting.
However, it is possible that $F_r^* > 0$ and retain utility degrades.
This inadvertent erasure of knowledge is also observed in other unlearning works~\citep{Nguyen2020-an, tian-etal-2024-forget}.
We term this phenomenon of improving forget quality causing retain utility degradation as \emph{collateral forgetting}.
To identify the per-iteration condition under which collateral forgetting is unavoidable, we investigate an alternative formulation that also imposes a constraint on the retain utility:
\begin{equation} 
\label{problem:impossible-first-order}
    \gr \cdot \dw \leq 0, \quad 
    \gf \cdot \dw \geq Q 
    \quad \text{and} \quad
    \norm{\dw} \leq R.
\end{equation}
Local collateral forgetting is unavoidable when it is infeasible to find a solution $\dw$ to Problem~\eqref{problem:impossible-first-order}.
The following theorem establishes the exact boundary for this infeasibility.
\begin{theorem}[Local Collateral Forgetting Condition]
\label{theorem:cf-first-order}
    Let $\kappa_2 \triangleq \sqrt{\lp \norm{\gr} \norm{\gf} \rp ^2 - Q^2 \norm{\gr}^2 / {R^2}}$.
    Problem~\eqref{problem:impossible-first-order} admits no feasible solution for $\dw$ if and only if $\kappa > \kappa_2$, implying that local collateral forgetting is unavoidable.
\end{theorem}
The infeasibility condition is again governed by the hardness measure $\kappa$.
When the problem is too hard~($\kappa > \kappa_2$), it becomes impossible to satisfy the forget quality constraint without compromising the retain utility.
By leveraging this condition, model owners can identify unavoidable local collateral forgetting prior to the unlearning process.
To better understand this condition, we revisit the extreme example when $D_r = D_f$, in which collateral forgetting is clearly unavoidable.
Theorem~\ref{theorem:cf-first-order} precisely captures this as the hardness $\kappa$ exceeds the threshold:
\begin{equation*}
    \kappa = \norm{\gr} \norm{\gf}  >  \sqrt{\norm{\gr}^2 \lp \norm{\gf}^2 - \frac{Q^2}{R^2} \rp} = \kappa_2.
\end{equation*}
The threshold $\kappa_2$ provides a clear interpretation.
The maximum dot product $\norm{\gr} \norm{\gf}$ represents the maximum hardness of the problem.
The condition implies that if the actual hardness is too close to this maximum, then Problem~\eqref{problem:impossible-first-order} becomes infeasible.
The term $- Q^2 \norm{\gr}^2 / {R^2}$ serves as a safety margin: 
a more negative value results in a lower threshold $\kappa_2$ that is more easily exceeded by $\kappa$. 
For example, when $Q$ is large, we place a stricter requirement on the forget quality, which makes it harder to satisfy both forget and retain constraints simultaneously.
The thresholds $\kappa_1$ and $\kappa_2$ partition the unlearning problem into distinct regimes of hardness, as illustrated in Fig.~\ref{fig:hardness_regimes}.

\section{The \ours Algorithm}
\label{sec:algo}
Building on the theoretical analysis in Sec.~\ref{sec:theory}, we propose the \emph{hardness-aware multi-objective unlearning}~(\ours) algorithm.
The core of \ours relies on Proposition~\ref{proposition:update-ldr} for the unlearning update and Theorem~\ref{theorem:cf-first-order} for the stopping criterion.
These results imply that after $T$ iterations, our \ours algorithm can improve the constraint by at least $TQ$, as long as linear approximation holds.

\textbf{Stochastic Update.}
As computing the full gradients $\gr$ and $\gf$ can be computationally expensive, we sample a batch of examples from $D_r$ and $D_f$ at each iteration and use the batch gradient estimates $\grb$ and $\gfb$ to update the weights.
The hardness measure~($\kappa$) and thresholds~($\kappa_1$, $\kappa_2$) are estimated using these batch gradients for the remainder of this section.
Based on Proposition~\ref{proposition:update-ldr}, at each iteration, we first estimate the hardness $\kappa$, and then apply either the direct update or the rectified update based on the threshold $\kappa_1$.
Following convention in the literature, we use the learning rate $\eta$ as our hyperparameter.
Specifically, we set $R = \eta \norm{\grb}$, under this choice of $R$, the direct update coincides with stochastic gradient descent on $D_r$.
We also investigate the effects of varying batch size in App.~\ref{app:ablation_batch_size} and observe that our method is relatively robust to batch size.

\textbf{Stopping Criterion.}
Before applying the unlearning update, we first compare the hardness $\kappa$ to the threshold $\kappa_2$, and terminate the algorithm if the local collateral forgetting condition is met~($\kappa > \kappa_2$).
Theorem~\ref{theorem:cf-first-order} suggests that the terminated algorithm cannot improve the approximate forget quality by $Q$ without degrading the approximate retain utility.
When $Q=0$, a local approximate Pareto optimum is reached.
This convergence to local optima is inherent to first-order methods and is commonly observed in deep learning optimization~\citep{pmlr-v49-lee16}.
These local minima are typically easy to find and well-performing~\citep{DBLP:conf/aistats/ChoromanskaHMAL15}, as demonstrated by our experiment results.
We discuss practical adaptations of our stopping criterion in App.~\ref{app:stopping_criterion}.

\textbf{Objective Variants.}
We propose the \primal and the \reciprocal algorithms based on two constrained optimization formulations.
The former is based on Problem~\eqref{problem:opt-ldr} while the latter is based on Problem~\eqref{problem:opt-ldf} where we optimize for the best forget quality while requiring the improvement on the retain utility to exceed $U$~(App.~\ref{app:retain-constrained}).
The pseudocode for \primal and \reciprocal are provided in Alg.~\ref{alg:hamu-q} and Alg.~\ref{alg:hamu-u}, respectively.
Here, $Q$ or $U$ are user-chosen parameters depending on their minimum requirements on forget quality or retain utility, and the only hyperparameter of our algorithms is the learning rate $\eta$.
We provide guidance on how to select $Q$ and $U$ in Sec.~\ref{sec:cv-exp}.

\textbf{Practical Applicability.}
In App.~\ref{app:first-order-loss-approximation}, we discuss in mathematical details when \ours first-order approximations are applicable to complex non-convex models. 
The approximation error is negligible when the learning rate and the maximum Hessian eigenvalue are small.
In App.~\ref{app:second-order}, we also discuss the second-order formulation which requires a convexity assumption and computation of the Hessian and may be impractical for large ML models.
Furthermore, \ours redoes the approximation at each iteration, \emph{thereby ensuring consistent approximation along the unlearning trajectory}.
\ours only introduces minimal computational overhead~(gradient dot product calculation) beyond standard training and is highly parallelizable due to its layer-wise constraint independence~(Sec.~\ref{sec:layer_specific_constraint}).
We report the GPU memory and computational cost in App.~\ref{app:cv_unlearning_setup}.

\subsection{Parallel and Practical Implementation}
\label{sec:layer_specific_constraint}
Our algorithm scales efficiently to large models by exploiting layer-wise independence, enabling straightforward multi-GPU parallelization.
We decompose the flattened weights $\vw \in \bR^d$ into $\ell$ components corresponding to the layers of the model, such that $\vw = [\vw^{(1)}, \ldots, \vw^{(\ell)}]$, where $\vw^{(i)} \in \bR^{d_i}$ denotes the weights of layer $i$.
The gradients $\gf$, $\gr$ and the weight update $\dw$ are partitioned similarly.
We distribute the global forget quality constraint $Q$ into layer-specific constraints $(Q_i)_{i=1}^{\ell}$ such that $\sum_{i=1}^{\ell} Q_i = Q$.
Enforcing the local condition $\gf^{(i)} \cdot \dw^{(i)} \geq Q_i$ guarantees compliance with the global constraint:
\begin{equation*}
    \gf \cdot \dw = \sum_{i=1}^{\ell} \gf^{(i)} \cdot \dw^{(i)} \geq \sum_{i=1}^{\ell} Q_i = Q.
\end{equation*}
This allows us to perform layer-wise weight update $\dw^{(i)}$ in Eq.~\eqref{eq:update-ldr} using layer-specific gradients $\gf^{(i)}$ and $\gr^{(i)}$.
Hence, our algorithm can be efficiently parallelized by assigning different layers to different GPUs.
More details about the parallelized implementation can be found in App.~\ref{app:parallelized}.
We also discuss how to improve the efficiency of \ours by using optimizers and practical compatibility of \ours with modern deep learning frameworks in App.~\ref{app:practical_considerations}.

Practitioners have observed that different layers in the same model encode qualitatively distinct information,
both for \emph{convolutional neural networks}~(CNNs) and LLMs~\citep{NIPS2017_dc6a7e65, pmlr-v267-skean25a}.
Consequently, these layers exhibit varying sensitivities to the same data.
To leverage this insight, instead of assigning each layer the same uniform constraint $Q / \ell$, we set the layer-wise constraint $Q_i$ to be proportional to the product of the gradient norms,
\begin{equation*}
    Q_i \triangleq \frac{\norm{\gr^{(i)}}\norm{\gf^{(i)}}}{\sum_{j=1}^{\ell} \norm{\gr^{(j)}}\norm{\gf^{(j)}}} \cdot Q.
\end{equation*}
Intuitively, a layer $i$ with larger gradient norms is more sensitive to optimization objectives and can contribute more to improving the forget quality, thus, should be assigned a higher $Q_i$.
We demonstrate how this technique further improves unlearning performance in Sec.~\ref{sec:cv-exp}.

\begin{algorithm}[t]
\caption{\primal}
\label{alg:hamu-q}
\begin{algorithmic}[1]
\STATE {\bfseries Input:} initial weights $\vw_0$, retain data $D_r$, forget data $D_f$, batch size $m$, learning rate $\eta$, maximum iterations $T$, forget quality constraint $Q>0$
\FOR{$t = 0,1,\dots,T-1$}
    \STATE Sample i.i.d.~size-$m$ batches $B_r \!\sim\! D_r$ and $B_f \!\sim\! D_f$
    \STATE $\grb \gets \nabla_{\vw} L_t(B_r); \ \gfb \gets \nabla_{\vw} L_t(B_f)$
    \STATE $R \gets \eta \norm{\grb}$
    \IF{$Q > R \norm{\gfb}$}
        \STATE \textbf{break}
    \ENDIF
    \STATE $\bar{\kappa} \gets \grb \cdot \gfb$
    \STATE $\bar{\kappa}_1 \gets -{Q \norm{\grb}} / {R}$
    \STATE $\bar{\kappa}_2 \gets \sqrt{\lp \norm{\grb} \norm{\gfb} \rp ^2 - {Q^2 \norm{\grb}^2} / {R^2}}$
    \IF{$\bar{\kappa} > \bar{\kappa}_2$}
        \STATE \textbf{break}
    \ENDIF
    \IF{$\bar{\kappa} \le \bar{\kappa}_1$}
        \STATE $\dw \gets - \dfrac{R}{\norm{\grb}} \cdot \grb$
    \ELSE
        \STATE $\grbp \gets \grb - \dfrac{\bar{\kappa}}{\norm{\gfb}^2} \cdot \gfb$
        \STATE $\dw \gets \dfrac{Q}{\norm{\gfb}^2} \cdot \gfb - \sqrt{R^2 - \dfrac{Q^2}{\norm{\gfb}^2}} \cdot \dfrac{\grbp}{\norm{\grbp}}$
    \ENDIF
    \STATE $\vw_{t+1} \gets \vw_t + \dw$
\ENDFOR
\STATE \textbf{return} $\vw_t$
\end{algorithmic}
\end{algorithm}

\section{Experiments}
\label{sec:exp}
In this section, we demonstrate the effectiveness of \ours and the validity of the hardness measure $\kappa$.
We compare \primal and \reciprocal with various unlearning baselines, including fine-tuning on the retain data~(FT), gradient ascent on the forget data~(GA)~\citep{jang2023knowledge}, gradient difference~(GDiff)~\citep{liu2022backdoor}, KL minimization~(KL)~\citep{maini2024tofu} and SCRUB~\citep{Kurmanji2023-xm}.
Details on these algorithms can be found in App.~\ref{app:baselines}.

We consistently use $\Delta L_f$ and $-\Delta L_r$ to denote the change in forget quality (forget loss) and retain utility (negated retain loss).
We mainly evaluate the unlearning performance based on the trade-off between $\Delta L_f$ and $-\Delta L_r$.
Our plots show $5$ unlearning epochs, with markers indicating the evaluation at the end of each epoch.
The stopping criterion was not enforced in some of the experiments to illustrate results even for the hardest unlearning settings.

In practice, we observe that the gradient norms $\norm{\gr}$ and $\norm{\gf}$ are usually larger than $1$.
Hence, we adopt the gradient clipping technique such that $\norm{\gr}_{\max} = \norm{\gf}_{\max} = 1$.
We then choose $Q$ such that $Q < \eta \norm{\gr}_{\max} \norm{\gf}_{\max}$.
This ensures that the assumption of Theorem~\ref{theorem:opt-primal} is generally satisfied in practice.
The same is done for $U$.
Unless otherwise stated, our experiments set $Q = U = 0.5 \eta$ for \primal and \reciprocal, respectively.
The same learning rate of $\num{1e-4}$ is used for all unlearning algorithms.

\subsection{Computer Vision Task}
\label{sec:cv-exp}
We first evaluate \ours on a \emph{computer vision}~(CV) task.
To obtain the original model $\vw_0$, we randomly initialize a ResNet-20~\citep{he2016deep} model and train it on the CIFAR-10~\citep{krizhevsky2009learning} dataset.
Details of the CV unlearning experimental setup are in App.~\ref{app:cv_exp_setup}.

\textbf{Varying hardness by changing similarity mixing ratio.}
To construct unlearning problems with different levels of hardness, we vary the similarity between the forget and retain data using the similarity mixing ratio $\rho$.
Specifically, we first choose a random class $c$ which contains $n$ samples.
Given a similarity mixing ratio $\rho \in \lb 0, 1 \rb$, the forget data consists of $n - \floor{\rho n}$ randomly sampled data points from class $c$, and $\floor{\rho n}$ randomly sampled data points from the union of all other classes and data points in class $c$ that are not selected in the first step.
When $\rho = 0$, the forget data is dissimilar to the retain data and data points from class $c$ appear only in the forget data.
In contrast, when $\rho = 1$, the forget data is highly similar to the retain data, and both contain data points from all classes.
In our experiments, we use $\rho \in \lc 0, 0.25, 0.5, 0.75, 1 \rc$.
We use the default $\rho = 0$ in our experiments unless otherwise stated.
\begin{figure}[t]
    \centering
    \includegraphics[width=1\linewidth]{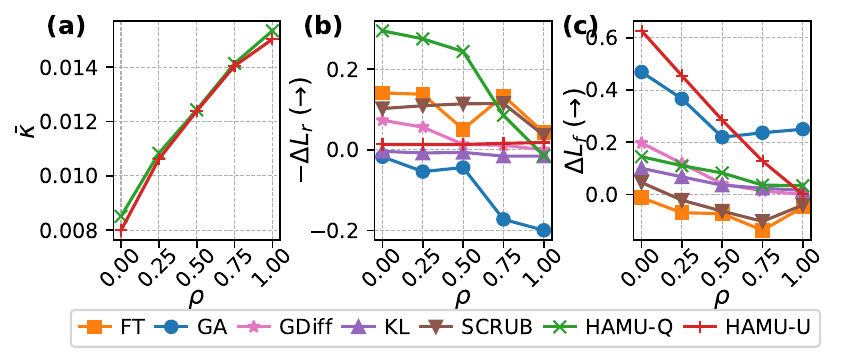}
    \caption{(a) Estimated hardness $\bar{\kappa}$ is strongly correlated with similarity mixing ratio $\rho$, with Pearson correlation coefficient of 0.994 and 0.986 for \primal and \reciprocal, respectively.
    (b, c) Higher $\rho$~(harder) results in a lower increase in $-\Delta L_r$ and $\Delta L_f$.}
    \label{fig:exp_cv_correlation}
\end{figure}

\textbf{Unlearning is harder when the forget and retain data are more similar.}
Fig.~\ref{fig:exp_cv_correlation}a shows that the estimate of our hardness measure, $\bar{\kappa}$, is highly correlated with the similarity mixing ratio $\rho$. 
Additionally, Fig.~\ref{fig:exp_cv_correlation}b-c shows that as unlearning hardness increases with $\rho$, the achievable improvements in retain utility $-\Delta L_r$ and forget quality $\Delta L_f$ are lower.
This verifies our intuition that unlearning becomes harder when forget data is more similar to retain data.
Notably, although the hardness measure is derived from \ours's formulation, the same trend is consistently observed across other unlearning algorithms~(Fig.~\ref{fig:exp_cv_correlation}b-c), suggesting that the similarity between the forget and retain data captured by our measure reflects a fundamental characteristic of machine unlearning.
\begin{figure}[t]
    \centering
    \includegraphics[width=1\linewidth]{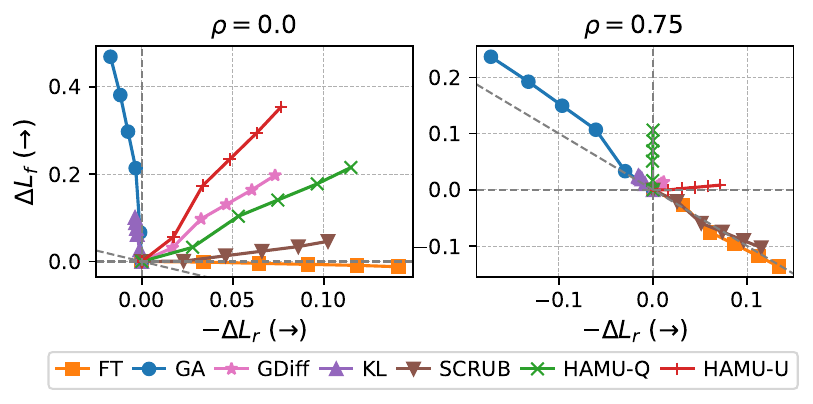}
    \caption{Comparison of \ours with baselines for unlearning scenarios of different levels of hardness.
    \ours outperforms baselines in both easy~($\rho = 0$) and hard~($\rho = 0.75$) scenarios.
    When unlearning becomes harder~(e.g.,~$\rho = 0.75$), most baseline methods fail to satisfy both objectives.}
    \label{fig:exp_cv_baseline}
\end{figure}

\textbf{\ours outperforms baselines, particularly when the problem is hard.}
In Fig.~\ref{fig:exp_cv_baseline}, we compare the unlearning trajectories of \ours against baselines for unlearning scenarios of different levels of hardness, controlled by the similarity mixing ratio $\rho$.
In these plots, all trajectories start from the origin, and the dotted diagonal line mimics a bad unlearning algorithm:
tracing the top-left of the dotted diagonal line~(e.g.,~GA, KL for $\rho=0.75$) indicates that model performance deteriorates on both the forget and retain data, resulting in collateral forgetting.
In contrast, tracing the bottom-right of the dotted diagonal line~(e.g.,~FT, SCRUB for $\rho=0.75$) indicates that model performance improves on both the forget and retain data, suggesting poor forget quality.
An ideal unlearning algorithm should achieve positive $\Delta L_f$ and $-\Delta L_r$, i.e.,~its unlearning trajectory should trend toward the top-right region.
In the easy scenario~($\rho = 0$ with dissimilar forget and retain data), GA and KL achieve a large $\Delta L_f$ with a negative $-\Delta L_r$,
FT and SCRUB achieve a large $-\Delta L_r$ with a negative $\Delta L_f$,
while \ours variants and GDiff desirably achieve large positive $\Delta L_f$ and $-\Delta L_r$, improving both objectives simultaneously.
However, in the hard scenario~($\rho = 0.75$ with similar forget and retain data), only \ours variants achieve non-negligible positive $\Delta L_f$ or $-\Delta L_r$ without hurting the other objective.
Additional experimental results for different levels of similarity mixing ratio $\rho$ are in App.~\ref{app:additional_baseline_cv}.
\begin{figure}[t]
    \centering
    \includegraphics[width=1\linewidth]{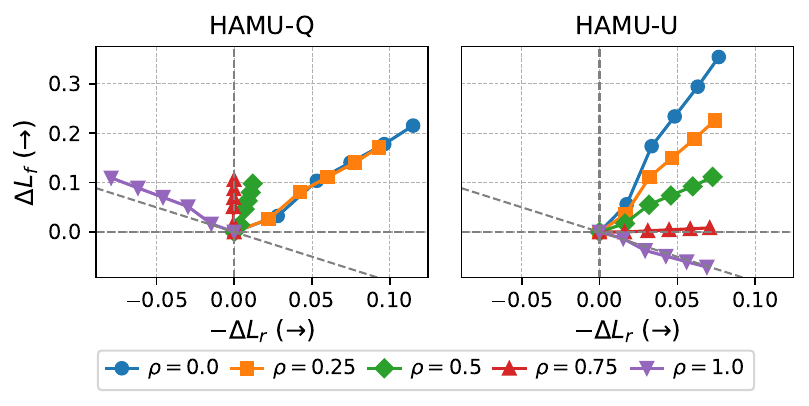}
    \caption{\primal enforces a positive $\Delta L_f$ while \reciprocal enforces a positive $-\Delta L_r$.}
    \label{fig:exp_cv_hardness}
\end{figure}

\textbf{\ours enforces the desired objectives.}
Fig.~\ref{fig:exp_cv_hardness} compares the unlearning trajectories of \primal and \reciprocal.
Across all $\rho$'s and epochs, we observe that \primal indeed enforces a positive $\Delta L_f$~(better forget quality), and \reciprocal enforces a positive $-\Delta L_r$~(better retain utility).
This is consistent with our formulation in Problem~\eqref{problem:opt-ldr} and Problem~\eqref{problem:opt-ldf}, respectively.
Fig.~\ref{fig:exp_cv_hardness} shows that as the forget data becomes more similar to the retain data~(higher $\rho$), $-\Delta L_r$ decreases for \primal and $\Delta L_f$ decreases for \reciprocal.
When $\rho=1$, unlearning becomes so hard that even \primal results in collateral forgetting and \reciprocal fails to unlearn the forget data.
Note that the stopping criterion was not enforced in these experiments to illustrate the results for these hard unlearning scenarios.
\begin{figure}[t]
    \centering
    \includegraphics[width=1\linewidth]{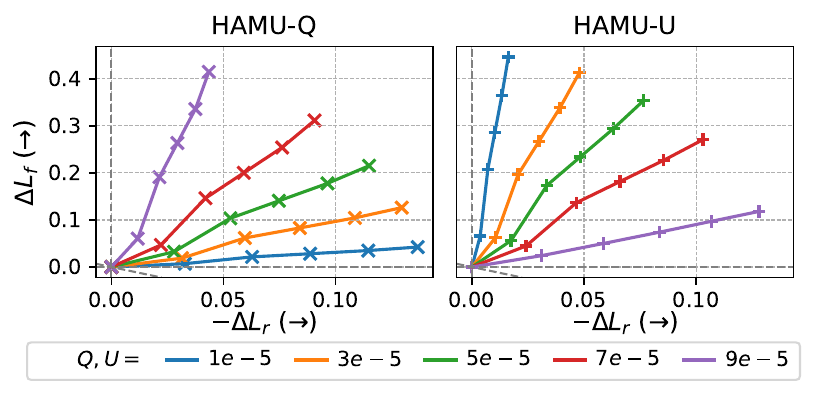}
    \caption{Unlearning trajectories of \primal and \reciprocal for different values of $Q$ and $U$ when $\rho=0$.}
    \label{fig:exp_cv_varying_Q_U}
\end{figure}

\textbf{Selection of constraints $Q$ and $U$ trades off between improvement in forget quality and retain utility.}
We illustrate the trade-off between $\Delta L_f$ and $-\Delta L_r$ for different values of $Q$ and $U$ in Fig.~\ref{fig:exp_cv_varying_Q_U}.
When using \primal, increasing $Q$ results in a better forget quality~(higher $\Delta L_f$), at the expense of a worse retain utility~(lower $-\Delta L_r$).
Conversely, when using \reciprocal, increasing $U$ results in a better retain utility, at the expense of a worse forget quality.
In practice, we leave the selection of $Q$ and $U$ to the user when using \ours based on their specific requirements.
\begin{figure}
    \centering
    \includegraphics[width=1\linewidth]{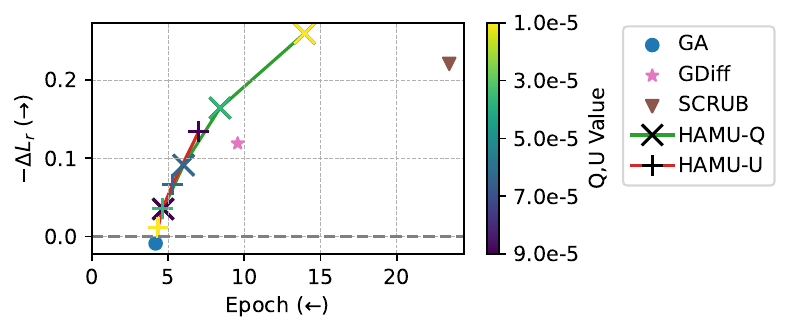}
    \caption{Plot of improvement in retain utility $-\Delta L_r$ against number of epochs required to reach a specified threshold for $L_f$.
    Methods that failed to reach this threshold have been excluded.}
    \label{fig:exp_cv_compute_trade_off}
\end{figure}

\textbf{Selection of constraints $Q$ and $U$ trades off between number of epochs and objective improvement.}
In Fig.~\ref{fig:exp_cv_compute_trade_off}, we plot $-\Delta L_r$ against the number of unlearning epochs required to reach a pre-specified target forget loss $L_f=1.0$.
We observe that higher $Q$ requires fewer epochs (i.e.,~lower computation cost) to reach the required forget quality, but leads to worse retain utility.
If more epochs can be afforded, the user can choose a smaller $Q$, in order to achieve a better retain utility.
Fig.~\ref{fig:exp_cv_compute_trade_off} also shows that the Pareto-front of \ours outperforms the baselines, with the Pareto-front of \primal overlapping with that of \reciprocal.

\textbf{Layer-wise update outperforms global update in unlearning performance.}
\begin{figure}[t]
    \centering
    \includegraphics[width=1\linewidth]{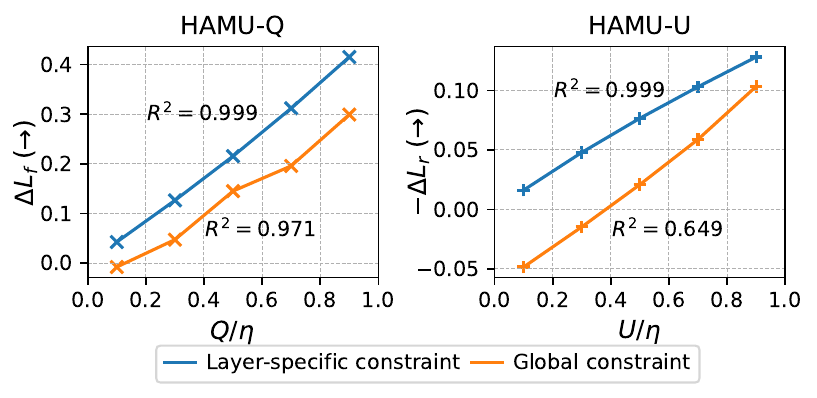}
    \caption{Plots of $\Delta L_f$ vs.~$Q/\eta$ and $-\Delta L_r$ vs.~$U/\eta$ under layer-wise and global constraint for \primal and \reciprocal, respectively.
    We assume linear proportionality through origin for $R^2$.}
    \label{fig:exp_cv_fitting_Q_U}
\end{figure}
In addition to the practical benefits discussed in Sec.~\ref{sec:layer_specific_constraint}, Fig.~\ref{fig:exp_cv_fitting_Q_U} shows that using layer-specific constraints in \ours weight updates achieves better forget quality and retain utility than using global constraint in the updates for \primal and \reciprocal, respectively.
Fig.~\ref{fig:exp_cv_fitting_Q_U} also shows that $\Delta L_f$ and $-\Delta L_r$ are undesirably negative for small $Q$ and $U$ under global constraint.
In addition, we observe an almost perfect linear proportional relationship~($R^2=0.999$) between the user-specified requirements, $Q$ or $U$, and the corresponding $\Delta L_f$ or $-\Delta L_r$ for \primal and \reciprocal, respectively.
This may be because the unlearning update is computed layer-wise, which allows a more accurate linear approximation of each layer's contribution to the overall loss, whereas the global update approach combines gradients across different layers and introduces additional nonlinearities.
\begin{figure}[t]
    \centering
    \includegraphics[width=1\linewidth]{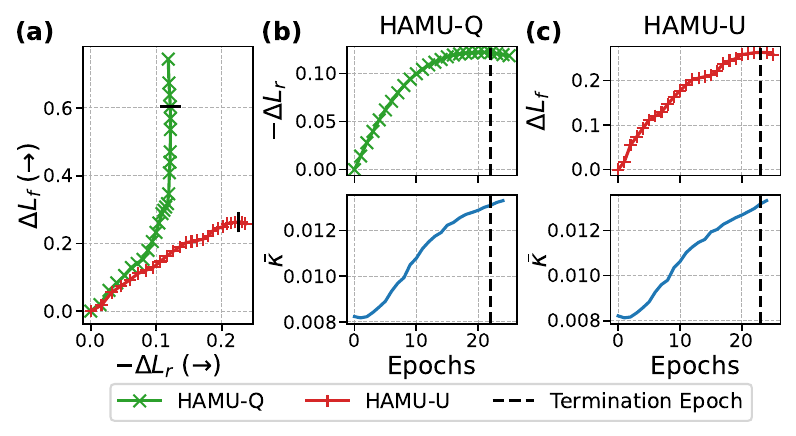}
    \caption{Unlearning trajectories for $25$ epochs at $\rho=0.5$, along with the estimated hardness $\bar{\kappa}$ at each epoch.
    Termination epoch refers to the epoch where the stopping criterion is met.}
    \label{fig:stopping_criterion}
\end{figure}

\textbf{Hardness increases as unlearning progresses.}
The bottom plots of Fig.~\ref{fig:stopping_criterion}b-c show that unlearning hardness increases with unlearning epochs.
This may be because as unlearning progresses, there is less room for maneuvering, and fewer update directions remain feasible for simultaneously improving both objectives.
The increase in hardness leads to a decrease in the rate of change for $-\Delta L_r$ and $\Delta L_f$ for \primal and \reciprocal, respectively, as shown in the top plots of Fig.~\ref{fig:stopping_criterion}b-c.

\textbf{Stopping criterion terminates \ours when collateral forgetting is inevitable.}
Fig.~\ref{fig:stopping_criterion}~b-c shows that once $-\Delta L_r$ and $\Delta L_f$ start to drop for \primal and \reciprocal, respectively, our stopping criterion terminates unlearning.
Note that this can happen at an epoch slightly after the turning point, as shown in the top plot of Fig.~\ref{fig:stopping_criterion}b, likely due to the linear approximation errors in \ours.
In practice, a model owner performing unlearning can avoid this by terminating unlearning when the stopping criterion is close to being met,
i.e.,~when $\bar{\kappa} > \bar{\kappa}_2 - \epsilon$ for some small $\epsilon > 0$.
In App.~\ref{app:stopping_criterion}, we discuss how to adapt the stopping criterion for layer-specific constraints.

\subsection{LLM Question Answering Task}
We additionally evaluate \ours on LLM \emph{question answering}~(QA) task, a more recent and larger scale unlearning setup.
We use the \texttt{Llama-2-7b-chat-hf} model\footnote{\url{https://huggingface.co/meta-llama/Llama-2-7b-chat-hf}} trained on the unlearning WaterDrum-TOFU dataset~\citep{lu2025waterdrum} which has subsets with different levels of similarity.
For our experiments, we use the ``No duplicate'' and ``Semantic duplicate'' subsets of WaterDrum-TOFU, where the latter includes paraphrased versions of the forget data in the retain data.
We refer to the settings as ``Semantically dissimilar'' and ``Semantically similar'', respectively.
Details on the training process are in App.~\ref{app:llm_unlearning_setup}.
We also include results for additional baselines and evaluation metrics such as model utility and \emph{membership inference attack}~(MIA) success rate in App.~\ref{app:additional_baselines_metrics}.
\begin{figure}
    \centering
    \includegraphics[width=1\linewidth]{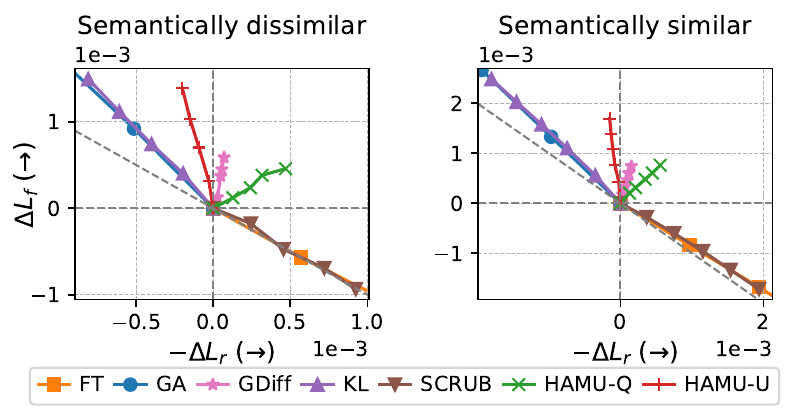}
    \caption{Comparison of \ours with baselines for different levels of semantic similarity on LLM QA task.
    The plot is zoomed in, with parts of FT and GA~(which perform poorly) cropped out.}
    \label{fig:exp_llm_hardness}
\end{figure}

\textbf{Unlearning is harder when the forget and retain data are more semantically similar.}
The average estimated hardness for the ``semantically similar'' setting~($\bar{\kappa} = \num{6.1e-4}$) is higher than that of the ``semantically dissimilar'' setting~($\bar{\kappa} = \num{4.0e-4}$), consistent with our intuition that it is easier to unlearn the forget data that is less similar to the retain data.
Fig.~\ref{fig:exp_llm_hardness} shows that \primal is able to improve both forget quality and retain utility for both settings.

\textbf{LLM unlearning is hard.}
In Fig.~\ref{fig:exp_llm_hardness}, only \primal and GDiff can improve both forget quality and retain utility simultaneously.
Comparatively, \primal is able to achieve a greater increase in forget quality and retain utility than GDiff.
Most baselines perform poorly, tracing close to the dotted diagonal which mimics a bad unlearning algorithm, even in the easier setting with semantically dissimilar forget and retain data.
Note that \reciprocal achieves a slightly negative $-\Delta L_r$ in Fig.~\ref{fig:exp_llm_hardness}, due to poor linear approximation,
i.e.,~the positive retain utility improvement constraint is satisfied by the approximated problem but not the original problem.
Fig.~\ref{fig:exp_learning_rate} shows that with a smaller learning rate, the retain objective improves from violating to satisfying the constraint due to better approximation.
\begin{figure}
    \centering
    \hspace*{-0.4cm}    
    \includegraphics[scale=.55]{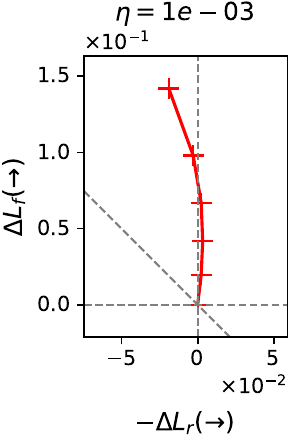}
    \hspace*{.8cm}
    \includegraphics[scale=.55]{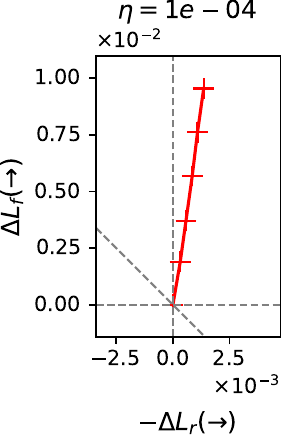}
    \caption{Comparison of \reciprocal unlearning trajectories for different learning rates $\eta$, with the same $U=0.6\eta$, under the semantically dissimilar LLM experimental setup.}
    \label{fig:exp_learning_rate}
\end{figure}

\section{Conclusion}
We introduce \ours, a novel hardness-aware multi-objective unlearning algorithm that adjusts updates based on per-iteration hardness and terminates when collateral forgetting is reached. 
Through theoretical analysis of a constrained optimization problem, \ours optimally reduces the cost of forgetting while satisfying the forget quality constraint.
Naturally arising from the analysis, we propose using the gradient dot product as the hardness measure to quantify the difficulty of reconciling the unlearning objectives. 
This proposed measure theoretically predicts the unlearning performance, formalizing the intuition that more similar data are harder to unlearn.
Our experiments demonstrate \ours's superior performance compared to existing baselines, and its applicability and scalability to large, non-convex models.
We believe \ours represents a significant step toward developing unlearning algorithms that are both practical and theoretically-grounded, enabling broader adoption of unlearning in real-world scenarios.
We discuss the relevant literature in App.~\ref{app:related_work}.

\section*{Acknowledgements}
This research is supported by the National Research Foundation, Singapore under its AI Singapore Programme (AISG Award No: AISG$3$-RP-$2022$-$029$).
This research/project is supported by the National Research Foundation, Singapore under its National Large Language Models Funding Initiative (AISG Award No: AISG-NMLP-$2024$-$001$).
Any opinions, findings and conclusions or recommendations expressed in this material are those of the author(s) and do not reflect the views of National Research Foundation, Singapore.
Jiangwei Chen is supported by the Institute for Infocomm Research (I\textsuperscript{2}R), Agency for Science, Technology and Research (A*STAR).
Xinyuan Niu is supported by the Centre for Frontier AI Research (CFAR), Agency for Science, Technology and Research, Singapore (A*STAR).
We would like to thank the anonymous reviewers and the AC for their helpful and constructive feedback.

\section*{Impact Statement}
This paper presents work whose goal is to advance the field of Machine Learning.
There are many potential societal consequences of our work, none which we feel must be specifically highlighted here.

\bibliography{icml2026}
\bibliographystyle{icml2026}

\newpage
\appendix
\onecolumn

\section{Related Work}
\label{app:related_work}
One common objective of machine unlearning is to approximate the retrained model~\citep{Georgiev2025-aj, Nguyen2022-it}, which is equivalent to minimizing the retain loss~\citep{Neel2021-bv, Bourtoule2021-tu}.
Another prevalent objective is to increase the forget loss~\citep{Halimi2022-aj, Zhang2024-yj}.
As neither of these objectives alone suffices to ensure a model with both strong forget and retain performance~\citep{wang2025rethinking, fan2025simplicity}, multi-objective unlearning~\citep{jin-etal-2025-unlearning} aims to optimize both objectives simultaneously.

One straightforward approach to multi-objective unlearning is to directly optimize a weighted combination of the loss functions~\citep{Kurmanji2023-xm, duan2025ready2unlearn}.
Other approaches such as GDR-GMA~\citep{Lin2024-wi} and GRU~\citep{wang2025gru} combine the retain gradient with the forget gradient to compute the unlearning update.
However, neither approach guarantees that a certain objective is achieved, nor do they adjust the update based on the hardness to optimally reduce the cost of forgetting.

Reconciling the retain and forget objectives in machine unlearning can be viewed as a specific instance of multi-objective optimization, and therefore multi-task learning approaches such as PCGrad~\citep{yu2020gradient} and CAGrad~\citep{NEURIPS2021_9d27fdf2} can be adapted for multi-objective unlearning.
Our approach is closely related to the classical $\epsilon$-constraint method~\citep{Haimes1971} but differs from standard approaches in that, instead of running gradient descent on the global constrained optimization problem, we derive a per-iteration, first-order update rule that enforces a minimum improvement in forget quality while minimizing retain utility degradation.
The novelty lies in this local, gradient-level formulation~(Sec.~\ref{sec:theory}), as well as in the hardness characterization and its interpretation, which directly inform the design of our algorithm~(Sec.~\ref{sec:algo}) and enable efficient application to large-scale, non-convex models.

Unlearning methods with theoretical guarantees~\citep{Guo2020-pp, Ullah2021-df} typically require strong assumptions, such as strongly convex loss functions and the existence of a unique minimizer~\citep{Wu2020-rf, Allouah2025-en}, limiting their practicality. 
Existing works relevant to unlearning hardness~\citep{Zhao2024-wq, Sarvmaili2024-bp} are largely heuristic-based and do not establish a formal connection between the proposed hardness measures and unlearning performance.

\section{Proofs and Derivations}
\label{app:proof}
\subsection{First-Order Loss Approximation}
\label{app:first-order-loss-approximation}
For any $\vz_r = (\vx_r, y_r) \in D_r$, applying Taylor expansion to $L_{t+1}(\vz_r)$, we have
\begin{equation*}
    L_{t+1}(\vz_r) = L_t(\vz_r) + \Delta \vw^\top\nabla L_t(\vz_r) + \underbrace{\frac12\Delta \vw^\top \mH \Delta \vw}_{E},  
\end{equation*}
where $E$ is the approximation error term, and $\mH = \nabla^2 L(\vp(\vx_r; \vw'), y_r)$ is the Hessian at some intermediate model $\vw'=\vw_t+c\Delta \vw$ for some scalar $c\in(0,1)$.
As the model update $\Delta \vw$ is induced by a single gradient descent step on $\vz_r$, substituting $\Delta \vw=-\eta\nabla L_t(\vz_r)$ gives
\begin{equation*}
    \Delta_{\vz_r} L_t(\vz_r) = L_{t+1}(\vz_r) - L_t(\vz_r) = -\eta \norm{\nabla L_t(\vz_r)}^2 + E.
\end{equation*}
The approximation quality can be measured by the ratio between the absolute values of the error term and linear term:
\begin{equation*}
\frac{|E|}{\abs{\Delta \vw^\top\nabla L_t(\vz_r)}} 
= \left|\frac{\frac12\Delta \vw^\top \mH \Delta \vw}{-\eta\norm{\nabla L_t(\vz_r)}^2}\right|
\leq\left|\frac{\frac12\eta^2\norm{\nabla L_t(\vz_r)}^2\eigenmax}{-\eta\norm{\nabla L_t(\vz_r)}^2}\right|=\frac12\eta\eigenmax,
\end{equation*}
where $\eigenmax$ is the maximum eigenvalue of $\mH$.
Empirical evidence~\citep{NEURIPS2018_905056c1, bui2024newton} has shown that the eigenvalues of Hessian typically falls below 200 for neural networks on standard datasets such as CIFAR-10.
Combined with a small learning rate of \num{1e-4} used in our experiments~(Sec.~\ref{sec:exp}), the error ratio falls below 1\%.
Therefore, the second-order derivative component is negligible in practice and using the first-order derivative component is sufficient in each iteration.
Furthermore, we are redoing the approximation at each iteration, thereby ensuring consistent approximation along the unlearning trajectory.
Since our formulation in Problem~\eqref{problem:opt-ldr} relies on approximation, the derived algorithms can be applied to any non-convex model as long as a sufficiently small learning rate is chosen to ensure approximation quality, the theoretical results also only hold up to the extent of approximation.
Despite this, our methods still perform well on non-convex models as demonstrated in our experiments~(Sec.~\ref{sec:exp}).

\subsection{The Forget-Constrained Problem}
\textbf{Theorem~\ref{theorem:opt-primal} and Proposition~\ref{proposition:update-ldr} are proven below by solving a convex optimization problem using KKT conditions.}

As Problem~\eqref{problem:opt-ldr} enforces a constraint on the forget quality, we refer to it as the forget-constrained problem.
Alternatively, we can enforce a constraint on the retain utility, we discuss this alternative retain-constrained problem in App.~\ref{app:retain-constrained}.
To solve Problem~\eqref{problem:opt-ldr}, first note that it is feasible when $Q \leq R \norm{\gf}$.
We assume this feasibility condition holds henceforth.
The Lagrangian of Problem~\eqref{problem:opt-ldr} is given by
\begin{equation*}
    \cL(\dw, \alpha, \beta) = \gr \cdot \dw + \alpha (Q - \gf \cdot \dw) + \beta (\dw \cdot \dw - R^2).
\end{equation*}
Since Problem~\eqref{problem:opt-ldr} is a convex optimization problem with convex inequality constraints, we can find its optimal value by solving the KKT conditions:
\begin{equation*}
    \begin{aligned}
        \textbf{Stationarity Condition}  &\qquad \pd{\cL}{\dw} = \gr - \alpha \gf + 2 \beta \dw = \vzero \\
        \textbf{Complementary Slackness} &\qquad \alpha (Q - \gf \cdot \dw) = 0 \\
                                         &\qquad \beta (\dw \cdot \dw - R^2) = 0 \\
        \textbf{Primal Feasibility}      &\qquad Q - \gf \cdot \dw \leq 0 \\
                                         &\qquad \dw \cdot \dw - R^2 \leq 0 \\
        \textbf{Dual Feasibility}        &\qquad \alpha \geq 0, \beta \geq 0.
    \end{aligned}
\end{equation*}
We discuss all possible cases corresponding to different values of the dual variables $\alpha$ and $\beta$.
Note that when the dual variables are positive, the corresponding primal feasibility conditions must bind. 
Specifically, when $\alpha > 0$, $\gf \cdot \dw = Q$; when $\beta > 0$, $\norm{\dw} = R$.

\fbox{\textbf{Case 1}: $\alpha = 0, \beta > 0$.} 

When $\alpha = 0$, the stationarity condition simplifies to 
\begin{equation*}
    \gr + 2 \beta \dw = \vzero, 
\end{equation*}
which implies
\begin{equation}
\label{eq:primal-case1-update}
    \dw = -\frac{1}{2\beta} \cdot \gr.
\end{equation}
To find the value of $\beta$, we equate the norm of the above expression to $R$ since $\norm{\dw} = R$ when $\beta > 0$:
\begin{equation*}
    \norm{\dw}^2 = \frac{1}{4\beta^2} \norm{\gr}^2 = R^2.
\end{equation*}
Solving for $\beta$, we obtain
\begin{equation*}
    \beta = \frac{\norm{\gr}}{2R}.
\end{equation*}
Substitute the value of $\beta$ into Eq.~\eqref{eq:primal-case1-update}, we obtain the direct update 
\begin{equation*}
    \dw^* = -\frac{R}{\norm{\gr}} \cdot \gr,
\end{equation*}
which gives the optimal value
\begin{equation*}
    F_r^* = \gr \cdot \dw^* = -R \norm{\gr}.
\end{equation*}
To show that the direct update corresponds to the case when $\kappa \leq \kappa_1$, we check that $\dw^*$ should satisfy primal feasibility:
\begin{equation*}
    \gf \cdot \dw^* \geq Q \implies \gr \cdot \gf \leq -\frac{Q\norm{\gr}}{R},
\end{equation*}
which is exactly $\kappa \leq \kappa_1$.

\fbox{\textbf{Case 2}: $\alpha > 0, \beta > 0$.}

We first decompose $\gr$ into components that are parallel and orthogonal to $\gf$.
Let $\gfu \triangleq \gf / \norm{\gf}$ be the unit vector of $\gf$,
then $\gr = x \cdot \gfu + \perpen{\gr}$, 
where $x = \gr \cdot \gfu$ is the magnitude of the parallel component, 
and $\perpen{\gr} \perp \gf$ is the orthogonal component.
Similarly we can decompose $\dw$ into $\dw = y \cdot \gfu + \perpen{\dw}$, where $\perpen{\dw} \perp \gf$.
When $\alpha > 0$, $\gf \cdot \dw = Q$, substitute the decomposed $\dw$ we obtain
\begin{equation*}
    \gf \cdot (y \cdot \gfu + \perpen{\dw}) = Q,
\end{equation*}
which gives the magnitude of $\dw$'s parallel component to $\gf$,
\begin{equation*}
    y = \frac{Q}{\norm{\gf}}.
\end{equation*}
By Pythagorean theorem, the magnitude of $\dw$'s orthogonal component is
\begin{equation*}
    \norm{\perpen{\dw}}^2 = \norm{\dw}^2 - y^2 = R^2 - \frac{Q^2}{\norm{\gf}^2}.
\end{equation*}
From the stationarity condition we know $\dw \in \operatorname{span} \{\gr, \gf\}$, thus $\perpen{\dw}$ is parallel to $\perpen{\gr}$.
As we want to minimize $\gr \cdot \dw$, we will choose $\perpen{\dw}$ in the opposite direction of $\perpen{\gr}$ so that $\perpen{\dw} \cdot \perpen{\gr} < 0$. 
Hence, the rectified update is
\begin{equation}
\label{eq:case2update}
    \dw^* = \frac{Q}{\norm{\gf}^2} \cdot \gf - \sqrt{R^2 - \frac{Q^2}{\norm{\gf}^2}} \cdot \frac{\perpen{\gr}}{\norm{\perpen{\gr}}}, 
\end{equation}
where $\perpen{\gr} = \gr - x \cdot \gfu = \gr - \frac{\gr \cdot \gf}{\norm{\gf}^2} \cdot \gf$. 
Thus, the optimal value $F_r^*$ is 
\begin{align}
    \gr \cdot \dw^* & = \frac{Q}{\norm{\gf}^2} \lp \gr \cdot \gf \rp - \norm{\perpen{\gr}} \sqrt{R^2 - \frac{Q^2}{\norm{\gf}^2}}
    \label{eq:case2opt} \\
    & = \frac{Q}{\norm{\gf}^2} \lp \gr \cdot \gf \rp - \sqrt{\lp \norm{\gr}^2 - \frac{(\gr \cdot \gf)^2}{\norm{\gf}^2} \rp \lp R^2 - \frac{Q^2}{\norm{\gf}^2} \rp}. \notag
\end{align}
Next, we will show how \textbf{Case 2 corresponds precisely to the scenario in which the boundary condition fails},
i.e.,~the rectified update is applied when $\kappa > \kappa_1$,
thus \textbf{justifying our use of hardness measure to inform the unlearning update}.

By decomposing the stationarity condition in the direction parallel to $\gf$ and the direction orthogonal to $\gf$,
\begin{equation*}
    x - \alpha \norm{\gf} + 2 \beta y = 0
    \quad \text{and} \quad
    \perpen{\gr} + 2 \beta \perpen{\dw} = \vzero,
\end{equation*}
we obtain the expressions for the dual variables:
\begin{equation*}
    \alpha = \frac{x + 2 \beta y}{\norm{\gf}}
    \quad \text{and} \quad 
    \beta = \frac{\norm{\perpen{\gr}}}{2 \norm{\perpen{\dw}}} .
\end{equation*}
From $\alpha > 0$, and substituting the value of $\beta$, we have
\begin{align}
    x + 2 \beta y &> 0 \notag \\
    x + 2 \cdot \frac{\norm{\perpen{\gr}}}{2 \norm{\perpen{\dw}}} \cdot y &> 0 \notag \\
    x \norm{\perpen{\dw}} + y \norm{\perpen{\gr}} &> 0. \label{ineq:primal-case2-condition}
\end{align}
If $\gr \cdot \gf > 0$, then clearly $\kappa = \gr \cdot \gf > -Q \norm{\gr} / R = \kappa_1$ and we are done.

If $\gr \cdot \gf < 0$, substituting the values of $x$, $y$ and norms $\norm{\perpen{\dw}}$, $\norm{\perpen{\gr}}$ into Ineq.~\eqref{ineq:primal-case2-condition},
\begin{align*}
    y \norm{\perpen{\gr}} &> - x \norm{\perpen{\dw}} \\
    \frac{Q}{\norm{\gf}} \sqrt{\norm{\gr}^2 - \frac{(\gr \cdot \gf)^2}{\norm{\gf}^2}} &> - \frac{\gr \cdot \gf}{\norm{\gf}} \sqrt{R^2 - \frac{Q^2}{\norm{\gf}^2}} \\
    \frac{Q^2}{\norm{\gf}^2} \lp \norm{\gr}^2 - \frac{(\gr \cdot \gf)^2}{\norm{\gf}^2} \rp &> \frac{(\gr \cdot \gf)^2}{\norm{\gf}^2} \lp R^2 - \frac{Q^2}{\norm{\gf}^2} \rp \\
    Q^2 \norm{\gr}^2 &> R^2 (\gr \cdot \gf)^2 \\
    Q \norm{\gr} &> - R \lp \gr \cdot \gf \rp \\
    \gr \cdot \gf &> - \frac{Q\norm{\gr}}{R}.
\end{align*}
Thus, Case 2 is precisely defined by $\gr \cdot \gf > - Q \norm{\gr} / R$, which is exactly $\kappa > \kappa_1$.

\textbf{For Cases 3 and 4, we first solve for the optimal values and updates, then show they are degenerate cases of Case 2.}

\fbox{\textbf{Case 3}: $\alpha = 0, \beta = 0$.}

When $\beta = 0$, the stationarity condition implies $\gr = \alpha \cdot \gf$.
Further, $\alpha = 0$ implies $\gr = \vzero$ so $\gr \cdot \dw = 0$ for any feasible $\dw$.
Hence, we will choose the minimum-norm solution that satisfies primal feasibility, 
i.e.,~$\dw^* = ( {Q}/{\norm{\gf}^2} ) \cdot \gf$.
The optimum $F_r^*$ is 0 since $\gr = \vzero$.

\fbox{\textbf{Case 4}: $\alpha > 0, \beta = 0$.}

From $\alpha > 0$, the primal feasibility must bind: $\gf \cdot \dw = Q$,
which puts $\dw$ at the boundary of the primal feasibility condition.
We will again choose the minimum-norm solution $\dw^* = ( {Q}/{\norm{\gf}^2} ) \cdot \gf$, 
then $\gr \cdot \dw^* = \kappa Q / {\norm{\gf}^2}$.

To show that the Cases 3 and 4 are subsumed by Case 2 ($\kappa > \kappa_1$), we need to show
\begin{itemize}
    \item $\gr \cdot \gf > -Q \norm{\gr} / R$ is satisfied under Cases 3 and 4;
    \item Consistency with Eq.~\eqref{eq:case2update} for $\dw^*$ under Cases 3 and 4;
    \item Consistency with Eq.~\eqref{eq:case2opt} for $\gr \cdot \dw^*$ under Cases 3 and 4.
\end{itemize}
In Case 3, $\gr \cdot \gf = 0 > -Q \norm{\gr} / R$.
Also, substituting $\gr = \vzero$ into the expression for $\perpen{\gr}$ leads to $\perpen{\gr} = \vzero$, so $\dw^*$ in Case 2 degenerates to $( {Q}/{\norm{\gf}^2} ) \cdot \gf$.
The optimal value in Case 2 also evaluates to $0$ when $\gr = \vzero$.

In Case 4, from $\gr = \alpha \cdot \gf$ we have $\alpha = (\gr \cdot \gf) / \norm{\gf}^2$.
Then $\alpha > 0$ implies $\gr \cdot \gf > 0 > -Q \norm{\gr} / R$.
The optimal update and value evaluate to $( {Q}/{\norm{\gf}^2} ) \cdot \gf$ and $(\gr \cdot \gf)Q / \norm{\gf}^2$, respectively, since $\perpen{\gr} = \vzero$ given the collinearity between $\gr$ and $\gf$.

Therefore, \textbf{our update rule and optimal value encompass all possible cases of Problem~\eqref{problem:opt-ldr}}.

\subsection{Relationship Between Unlearning Performance and the Hardness Measure}
\textbf{The proof of Corollary~\ref{corollary:hardness} is given below.}
\begin{proof}
    When $\kappa \leq \kappa_1$, the function is constant. 
    When $\kappa > \kappa_1$, taking the derivative of $F_r^*$ with respect to $\kappa$ yields
    \begin{equation*}
        \frac{\ud F_r^*}{\ud \kappa} = \frac{Q}{\norm{\gf}^2} + \frac{\kappa \sqrt{R^2 - Q^2 / \norm{\gf}^2}}{\norm{\gf}^2 \sqrt{\norm{\gr}^2 - \kappa^2 / \norm{\gf}^2}}.
    \end{equation*}
    Simplify the derivative further using $\kappa = \norm{\gr} \norm{\gf} \cos\theta$, where $\theta \in [0, \pi]$ is the angle between $\gr$ and $\gf$,
    \begin{equation*}
        \frac{\ud F_r^*}{\ud \kappa} = \frac{1}{\norm{\gf}^2} \lp Q + \cot\theta \sqrt{R^2 \norm{\gf}^2 - Q^2} \rp.
    \end{equation*}
    At the threshold $\kappa = \kappa_1 = - Q \norm{\gr} / R$, the derivative evaluates to $0$. 
    When $\kappa$ increases, for fixed $\norm{\gr}$ and $\norm{\gf}$, $\cos\theta$ increases, $\theta$ decreases, $\cot\theta$ increases.
    Hence, the derivative is positive and the function is monotonically increasing.
\end{proof}

\subsection{Unavoidable Collateral Forgetting~(First-Order)}
\label{app:cf-first}
\textbf{Theorem~\ref{theorem:cf-first-order} is proven below.}

Assume $Q \leq R \norm{\gf}$.
Problem~\eqref{problem:impossible-first-order} is infeasible if $F_r^*$~(defined in Eq.~\eqref{eq:opt-ldr}) is greater than $0$. 
When $\kappa \leq \kappa_1$, $F_r^* = - R \norm{\gr} < 0$ so Problem~\eqref{problem:impossible-first-order} is feasible.
When $\kappa > \kappa_1$, the problem is infeasible if
\begin{align}
    \frac{\kappa Q}{\norm{\gf}^2} &> \sqrt{\lp \norm{\gr}^2 - \frac{(\gr \cdot \gf)^2}{\norm{\gf}^2} \rp \lp R^2 - \frac{Q^2}{\norm{\gf}^2} \rp} \label{ineq:impossible-first-order} \\
    \kappa^2 Q^2  &> \lp \norm{\gr}^2 \norm{\gf}^2 - (\gr \cdot \gf)^2 \rp \lp R^2 \norm{\gf}^2 - Q^2 \rp \notag \\
    \kappa^2 &> \norm{\gr}^2 \norm{\gf}^2 - \frac{Q^2 \norm{\gr}^2}{R^2} \notag \\
    \kappa &> \sqrt{\norm{\gr}^2 \norm{\gf}^2 - \frac{Q^2 \norm{\gr}^2}{R^2}} \label{ineq:cf-threshold},
\end{align}
where Ineq.~\eqref{ineq:cf-threshold} is because $\kappa > 0$, which is implied by Ineq.~\eqref{ineq:impossible-first-order}.
This gives the local collateral forgetting condition.

When $\kappa < 0$, Ineq.~\eqref{ineq:impossible-first-order} trivially fails; 
when $0 < \kappa \leq \sqrt{\norm{\gr}^2 \norm{\gf}^2 - \frac{Q^2 \norm{\gr}^2}{R^2}}$, Ineq.~\eqref{ineq:impossible-first-order} also fails.
Therefore, the feasibility condition for Problem~\eqref{problem:impossible-first-order} is
\begin{equation*}
    \kappa \leq \sqrt{\norm{\gr}^2 \norm{\gf}^2 - \frac{Q^2 \norm{\gr}^2}{R^2}} .
\end{equation*}
When $\kappa > 0$, let $\theta \in [0, \pi / 2]$ be the angle between $\gr$ and $\gf$, we can further simplify the feasibility condition:
\begin{align*}
    \kappa = \gr \cdot \gf &\leq \sqrt{\norm{\gr}^2 \norm{\gf}^2 - \frac{Q^2 \norm{\gr}^2}{R^2}} \\
    \frac{\gr \cdot \gf}{\norm{\gr} \norm{\gf}} &\leq \frac{1}{\norm{\gr} \norm{\gf}} \sqrt{\norm{\gr}^2 \norm{\gf}^2 - \frac{Q^2 \norm{\gr}^2}{R^2}} \\
    \cos \theta &\leq \sqrt{1 - \frac{Q^2}{R^2 \norm{\gf}^2}} \\
    \frac{Q^2}{R^2 \norm{\gf}^2} &\leq 1 - \cos^2 \theta \\
    Q^2 &\leq R^2 \norm{\gf}^2 \sin^2 \theta \\
    Q &\leq R \norm{\gf} \sin \theta .
\end{align*}
Problem~\eqref{problem:impossible-first-order} is always feasible when $\kappa \leq 0$; 
when $\kappa > 0$, let $\theta \in [0, \pi / 2]$ be the angle between $\gr$ and $\gf$, then Problem~\eqref{problem:impossible-first-order} is feasible if $Q \leq R \norm{\gf} \sin\theta$, which is stricter than the feasibility condition of Problem~\eqref{problem:opt-ldr} due to the additional constraint on retain utility.

\subsection{Unavoidable Collateral Forgetting~(Second-Order)}
\label{app:second-order}
In this section, we discuss the local collateral forgetting condition in the second-order formulation.
Identifying this condition requires solving the second-order formulation of Problem~\eqref{problem:opt-ldr}, 
which also yields an update rule but remains impractical due to assumptions and computational cost involving the Hessian.
We find that it is also impossible to avoid collateral forgetting in the second-order formulation under certain scenarios.
In fact, the hardness measure again plays an important role in determining whether we can avoid local collateral forgetting in machine unlearning.
We show that \textbf{when the problem becomes too hard~($\kappa = \gr \cdot \gf$ exceeds a certain threshold), collateral forgetting is unavoidable}.

We formulate the problem similar to Problem~\eqref{problem:impossible-first-order}, but better approximates $\Delta L_t(D_r)$ with a second-order approximation, while keeping first-order approximation for $\Delta L_t(D_f)$ to ensure we can still obtain the solution in analytical form.
Assume the Hessian on the retain data $\Hr \triangleq \nabla_{\vw}^2 L_t(D_r)$ is positive definite, we wish (or prove impossible) to find $\dw$ such that the following constraints are simultaneously satisfied:
\begin{equation} 
\label{problem:impossible-second-order}
    \gr \cdot \dw + \frac{1}{2} \dw^\top \Hr \dw \leq 0
    \quad \text{and} \quad
    \gf \cdot \dw \geq Q.
\end{equation}
The first constraint requires the retain utility to not drop after the update, while the second constraint requires the forget quality to exceed a threshold $Q > 0$.
Note that unlike in Problem~\eqref{problem:impossible-first-order}, we do not constrain the maximum norm of the weight update, so if the problem is infeasible, then collateral forgetting is unavoidable for all possible weight updates $\dw$.
To find the infeasibility condition, we consider the following convex optimization problem, which is the second-order formulation of Problem~\eqref{problem:opt-ldr}:
\begin{equation} 
\label{problem:convex-second-order}
    \begin{aligned}
        & \underset{\dw}{\text{minimize}}  && F(\dw) \triangleq \gr \cdot \dw + \frac{1}{2} \dw^\top \Hr \dw \\
        & \text{subject to}                && \gf \cdot \dw \geq Q .
    \end{aligned}
\end{equation}
If the constraint in Problem~\eqref{problem:convex-second-order} is satisfied, and the minimum objective value $F(\dw^*) > 0$, then Problem~\eqref{problem:impossible-second-order} is infeasible.
This establishes the local collateral forgetting condition in the second-order formulation.
The Lagrangian of Problem~\eqref{problem:convex-second-order} is given by 
\begin{equation*}
    \cL(\dw, \gamma) = \gr \cdot \dw + \frac{1}{2} \dw^\top \Hr \dw + \gamma \lp Q - \gf \cdot \dw \rp .
\end{equation*}
Since Problem~\eqref{problem:convex-second-order} is a convex optimization problem with convex inequality constraints, we can find its optimal value by solving the KKT conditions:
\begin{equation*}
    \begin{aligned}
        \textbf{Stationarity Condition}  &\qquad \pd{\cL}{\dw} = \gr + \Hr \dw - \gamma \gf = \vzero \\
        \textbf{Complementary Slackness} &\qquad \gamma (Q - \gf \cdot \dw) = 0 \\
        \textbf{Primal Feasibility}      &\qquad Q - \gf \cdot \dw \leq 0 \\
        \textbf{Dual Feasibility}        &\qquad \gamma \geq 0.
    \end{aligned}
\end{equation*}
From the stationarity condition we have
\begin{equation} 
\label{eq:update-second-order}
    \dw^* = \gamma \Hr^{-1} \gf - \Hr^{-1} \gr .
\end{equation}
\fbox{\textbf{Case 1}: $\gamma = 0$.}

Then $\dw^* = - \Hr^{-1} \gr$.
We can verify that $\gr \cdot \dw^* + \frac{1}{2} \dw^{*\top} \Hr \dw^* \leq 0$ is always true:
\begin{equation*}
    \begin{aligned}
        & \gr \cdot \dw^* + \frac{1}{2} \dw^{*\top} \Hr \dw^* \\
        ={} & \gr \cdot \lp - \Hr^{-1} \gr \rp + \frac{1}{2} \lp - \Hr^{-1} \gr \rp^{\top} \Hr \lp - \Hr^{-1} \gr \rp \\
        ={} & - \gr^\top \Hr^{-1} \gr + \frac{1}{2} \gr^\top \Hr^{-1} \Hr \Hr^{-1} \gr \\
        ={} & - \frac{1}{2} \gr^\top \Hr^{-1} \gr \\
        \leq{} & 0,
    \end{aligned}
\end{equation*}
where we use $\lp \Hr^{-1} \rp^\top = \lp \Hr^\top \rp^{-1} = \Hr^{-1}$ because $\Hr$ is symmetric, and the inequality is due to $\Hr \succ 0$.
Therefore, Problem~\eqref{problem:impossible-second-order} is feasible if $\dw^*$ additionally satisfies $\gf \cdot \dw^* \geq Q$.
To simplify discussion, we first assume a spherical loss landscape,
i.e.,~$\Hr = \lambda \mI$ for some $\lambda > 0$, 
and discuss the general case afterwards.
Problem~\eqref{problem:impossible-second-order} is infeasible if
\begin{equation*}
    \begin{aligned}
        \gf \cdot \dw^* &< Q \\
        \gf \cdot \lp - \Hr^{-1} \gr \rp &< Q \\
        - \gf^\top \Hr^{-1} \gr &< Q \\
        - \gf^\top \lp \frac{1}{\lambda} \mI \rp \gr &< Q \\
        \gr \cdot \gf &> - \lambda Q.
    \end{aligned}
\end{equation*}
For the infeasibility condition of general $\Hr \!\succ\! 0$, we first establish a lower bound for $\gf^\top \Hr^{-1} \gr$ with the following lemma:
\begin{lemma} \label{lemma:bilinear-bound}
    Let $\eigenmin$, $\eigenmax$ be the minimum and maximum eigenvalues of a symmetric matrix $\mA_{d \times d}$, and $\vu$, $\vv$ are vectors of dimension $d$, then $\vu^\top \mA \vv \geq \eigenmin \vu \cdot \vv - \lp \eigenmax - \eigenmin \rp \norm{\vu} \norm{\vv}$.
\end{lemma}
\begin{proof}
    Decompose $\mA$ into $\mA = \eigenmin \mI + \lp \mA - \eigenmin \mI \rp$ and denote $\mB \triangleq \mA - \eigenmin \mI$.
    Then the spectral norm of $\mB$ is $\norm{\mB}_2 = \eigenmax - \eigenmin$. 
    We have
    \begin{align*}
        \vu^\top \mA \vv &= \vu^\top \lp \eigenmin \mI \rp \vv + \vu^\top \mB \vv \\
        &\geq \eigenmin \vu \cdot \vv - \norm{\vu} \norm{\mB \vv} 
        \tag{\text{Cauchy–Schwarz inequality}} \\
        &\geq \eigenmin \vu \cdot \vv - \norm{\vu} \norm{\mB}_2 \norm{\vv} 
        \tag{\text{Definition of spectral norm}} \\
        &= \eigenmin \vu \cdot \vv - \lp \eigenmax - \eigenmin \rp \norm{\vu} \norm{\vv}.
    \end{align*}
\end{proof}
Problem~\eqref{problem:impossible-second-order} is infeasible if $- \gf^\top \Hr^{-1} \gr < Q$, by applying Lemma~\ref{lemma:bilinear-bound} we obtain the infeasibility condition:
\begin{equation*}
    \gr \cdot \gf > \frac{1}{\eigenmin} \lb \lp \eigenmax - \eigenmin \rp \norm{\gr} \norm{\gf} - Q \rb,
\end{equation*}
where $\eigenmin$, $\eigenmax$ are the minimum and maximum eigenvalues of $\Hr^{-1}$, respectively.

\fbox{\textbf{Case 2}: $\gamma > 0$.}

Primal feasibility must bind so $\gf \cdot \dw^* = Q$.
Substituting the expression for $\dw^*$ in Eq.~\eqref{eq:update-second-order} and solving for $\gamma$,
\begin{equation*}
    \gamma = \frac{Q + \gf^\top \Hr^{-1} \gr}{\gf^\top \Hr^{-1} \gf}.
\end{equation*}
To simplify notation, let $A = \gf^\top \Hr^{-1} \gf$, $B = \gr^\top \Hr^{-1} \gr$, $C = \gr^\top \Hr^{-1} \gf = \gf^\top \Hr^{-1} \gr$, then $\gamma = \frac{Q + C}{A}$, and
\begin{align*}
    F(\dw^*)
    &= \gr \cdot \lp \gamma \Hr^{-1} \gf - \Hr^{-1} \gr \rp + \frac{1}{2} \lp \gamma \Hr^{-1} \gf - \Hr^{-1} \gr \rp ^\top \Hr \lp \gamma \Hr^{-1} \gf - \Hr^{-1} \gr \rp \\
    &= \gamma C - B + \frac{1}{2} \lp \gamma^2 A - 2 \gamma C + B \rp \\
    &= \frac{1}{2} \gamma^2 A - \frac{1}{2} B \\
    &= \frac{1}{2} \lp \frac{Q + C}{A} \rp^2 A - \frac{1}{2} B \\
    &= \frac{1}{2} \frac{(Q + C)^2}{A} - \frac{1}{2} B .
\end{align*}
Problem~\eqref{problem:impossible-second-order} is infeasible if $F(\dw^*) > 0$, which is equivalent to $(Q + C)^2 > AB$.
Note that $\frac{Q + C}{A} = \gamma > 0$ and $A > 0$ since $\Hr^{-1} \succ 0$ (assuming $\gf \neq \vzero$),
we know that $Q + C > 0$.
The infeasibility condition then becomes
\begin{equation} \label{ineq:infeasibility}
    Q + C > \sqrt{AB}.
\end{equation}
If we again assume that $\Hr = \lambda \mI$ for some $\lambda > 0$, then $A = \frac{1}{\lambda} \norm{\gf}^2$, $B = \frac{1}{\lambda} \norm{\gr}^2$, $C = \frac{1}{\lambda} \gr \cdot \gf$.
Thus,
\begin{align}
    Q + \frac{1}{\lambda} \gr \cdot \gf &> \sqrt{\frac{1}{\lambda} \norm{\gr}^2 \cdot \frac{1}{\lambda} \norm{\gf}^2} \notag \\
    \gr \cdot \gf &> \norm{\gr} \norm{\gf} - \lambda Q. \label{ineq:cf-spherical}
\end{align}
For the general case, again let $\eigenmin$, $\eigenmax$ be the minimum and maximum eigenvalues of $\Hr^{-1}$, by applying Lemma~\ref{lemma:bilinear-bound},
\begin{equation*}
    C = \gr^\top \Hr^{-1} \gf \geq \eigenmin \gr \cdot \gf - \lp \eigenmax - \eigenmin \rp \norm{\gr} \norm{\gf}.
\end{equation*}
By the Rayleigh quotient bound, $A \leq \eigenmax \norm{\gf}^2$, $B \leq \eigenmax \norm{\gr}^2$, thus
\begin{equation*}
    \sqrt{AB} - Q \leq \sqrt{\eigenmax \norm{\gf}^2 \cdot \eigenmax \norm{\gr}^2} - Q = \eigenmax \norm{\gr} \norm{\gf} - Q.
\end{equation*}
From Ineq.~\eqref{ineq:infeasibility} we know the infeasibility condition is 
\begin{equation*}
    C > \sqrt{AB} - Q.
\end{equation*}
Therefore, the problem is infeasible if the lower bound of $C$ is strictly greater than the upper bound of $\sqrt{AB} - Q$,
\begin{align*}
    \eigenmin \gr \cdot \gf - \lp \eigenmax - \eigenmin \rp \norm{\gr} \norm{\gf} &> \eigenmax \norm{\gr} \norm{\gf} - Q \\
    \eigenmin \gr \cdot \gf &> 2\eigenmax \norm{\gr} \norm{\gf} - \eigenmin \norm{\gr} \norm{\gf} - Q \\
    \gr \cdot \gf &> \frac{1}{\eigenmin} \lb \lp 2\eigenmax - \eigenmin \rp \norm{\gr} \norm{\gf} - Q \rb . \\
\end{align*}
Observe that \textbf{we can combine both cases} since the condition when $\lambda = 0$ is subsumed by the condition when $\lambda > 0$:
\begin{equation*}
    \frac{1}{\eigenmin} \lb \lp 2\eigenmax - \eigenmin \rp \norm{\gr} \norm{\gf} - Q \rb \geq \frac{1}{\eigenmin} \lb \lp \eigenmax - \eigenmin \rp \norm{\gr} \norm{\gf} - Q \rb ,
\end{equation*}
and
\begin{equation*}
    \norm{\gr} \norm{\gf} - \lambda Q \geq - \lambda Q.
\end{equation*}
We interpret the local collateral forgetting condition under the assumption of a spherical loss landscape~(Ineq.~\eqref{ineq:cf-spherical}).
The maximum dot product $\norm{\gr} \norm{\gf}$ describes the maximum hardness of the problem, the condition says that if the actual hardness of the problem is too close to this maximum, then collateral forgetting is unavoidable.
The term $-\lambda Q$ acts as a safety margin, lowering which lowers the threshold and it is easier for the problem to become infeasible.
If $\lambda$ is large, the loss surface is steep, it is harder to bound the decrease in retain utility.
If $Q$ is large, we are putting a stricter demand on the forget quality, which makes it harder to satisfy both constraints simultaneously.
In the extreme case when $\lambda Q$ is large enough, it is possible that $\norm{\gr} \norm{\gf} - \lambda Q < 0$, and even a negative $\gr \cdot \gf$ could theoretically be infeasible.

When the problem is feasible, i.e.,~both constraints are satisfied, Eq.~\eqref{eq:update-second-order} provides an unlearning update that does not induce collateral forgetting.
However, it requires computing and inverting the Hessian, which is prohibitively expensive.
It also requires the additional assumption that the Hessian is positive definite, hence limiting its applicability to non-convex models, whereas our first-order algorithm can be applied to any non-convex model as long as a suitable $R$ is chosen for approximation quality.

\subsection{The Retain-Constrained Problem}
\label{app:retain-constrained}
The retain-constrained problem provides an alternative formulation to Problem~\eqref{problem:opt-ldr}.
In Problem~\eqref{problem:opt-ldf}, we optimize for the best forget quality improvement while also enforcing a constraint on the minimum retain utility improvement.
The retain-constrained problem is formulated below for $U > 0$, $R > 0$:
\begin{equation} 
\label{problem:opt-ldf}
\begin{aligned}
    & \underset{\dw}{\text{maximize}}  && \gf \cdot \dw \\
    & \text{subject to}                && \gr \cdot \dw \leq -U, \quad \norm{\dw} \leq R .
\end{aligned}
\end{equation}
which is the same as
\begin{equation*}
\begin{aligned}
    & \underset{\dw}{\text{minimize}}  && -\gf \cdot \dw \\
    & \text{subject to}                && -\gr \cdot \dw \geq U, \quad \norm{\dw} \leq R . 
\end{aligned}
\end{equation*}
This is structurally equivalent to Problem~\eqref{problem:opt-ldr}, and is only feasible when $U \leq R \norm{\gr}$.
We can derive the solution to Problem~\eqref{problem:opt-ldf} by substituting $\gr \leftarrow -\gf, \gf \leftarrow -\gr, Q \leftarrow U$ into the solution to Problem~\eqref{problem:opt-ldr}, and then negating the optimal value.
Note that the hardness measure $\kappa = \gr \cdot \gf$ is preserved under this substitution:
\begin{equation*}
\begin{aligned}
    \dw^* &= 
    \begin{dcases}
        - \frac{R}{\norm{-\gf}} \cdot (-\gf) \quad & \text{if } \kappa \leq - \frac{U \norm{-\gf}}{R} \\
        \frac{U}{\norm{-\gr}^2} \cdot (-\gr) - \sqrt{R^2 - \frac{U^2}{\norm{-\gr}^2}} \cdot \frac{-\perpen{\gf}}{\norm{-\perpen{\gf}}} \quad & \text{otherwise}
    \end{dcases} \\
    &= 
    \begin{dcases}
        \frac{R}{\norm{\gf}} \cdot \gf \quad & \text{if } \kappa \leq \kappa_3 \\
        \mathrlap{-\frac{U}{\norm{\gr}^2} \cdot \gr + \sqrt{R^2 - \frac{U^2}{\norm{\gr}^2}} \cdot \frac{\perpen{\gf}}{\norm{\perpen{\gf}}}}
        \hphantom{\frac{U}{\norm{-\gr}^2} \cdot (-\gr) - \sqrt{R^2 - \frac{U^2}{\norm{-\gr}^2}} \cdot \frac{-\perpen{\gf}}{\norm{-\perpen{\gf}}}}
        \quad & \text{otherwise},
    \end{dcases}
\end{aligned}
\end{equation*}
where $\perpen{\gf} \triangleq \gf - \frac{\gr \cdot \gf}{\norm{\gr}^2} \cdot \gr$ and $\kappa_3 \triangleq -U \norm{\gf} / {R}$.
Similarly, we have
\begin{equation*}
    \min (-\gf \cdot \dw) = 
    \begin{dcases}
        - R \norm{\gf} \quad & \text{if } \kappa \leq \kappa_3 \\
        \frac{\kappa U}{\norm{\gr}^2} - \sqrt{\lp \norm{\gf}^2 - \frac{\kappa^2}{\norm{\gr}^2} \rp \lp R^2 - \frac{U^2}{\norm{\gr}^2} \rp} \quad & \text{otherwise}.
    \end{dcases}
\end{equation*}
The optimal value of the original problem is then its negation:
\begin{equation*}
    \max \gf \cdot \dw = 
    \begin{dcases}
        R \norm{\gf} \quad & \text{if } \kappa \leq \kappa_3 \\
        - \frac{\kappa U}{\norm{\gr}^2} + \sqrt{\lp \norm{\gf}^2 - \frac{\kappa^2}{\norm{\gr}^2} \rp \lp R^2 - \frac{U^2}{\norm{\gr}^2} \rp} \quad & \text{otherwise}.
    \end{dcases}
\end{equation*}
This also implies that $\max \gf \cdot \dw$ monotonically decreases with $\kappa$ for fixed $\norm{\gr}$ and $\norm{\gf}$, given $F_r^* = \min \gr \cdot \dw$ monotonically increases with $\kappa$ for fixed $\norm{\gr}$ and $\norm{\gf}$, and the fact that hardness $\kappa$ is preserved under our substitution.

The direct update ($\kappa \leq \kappa_3$) in the retain-constrained problem corresponds to gradient ascent on the forget data if we set $R = \eta \norm{\gf}$.
To find the stopping criterion of \reciprocal, we find the infeasibility condition to the following problem:
\begin{equation} 
\label{problem:cf-first-reciprocal}
    \gf \cdot \dw \geq 0, \quad 
    \gr \cdot \dw \leq -U 
    \quad \text{and} \quad
    \norm{\dw} \leq R.
\end{equation}
Similar to App.~\ref{app:cf-first}, we cannot find a feasible $\dw$ to Problem~\eqref{problem:cf-first-reciprocal} if $\max \gf \cdot \dw < 0$.
The stopping criterion is obtained by solving the inequality. 
Consequently, we terminate \reciprocal when $\kappa > \kappa_4$, where 
\begin{equation*}
    \kappa_4 \triangleq \sqrt{\norm{\gr}^2 \norm{\gf}^2 - \frac{U^2 \norm{\gf}^2}{R^2}}.
\end{equation*}
We provide the pseudocode to \reciprocal in Alg.~\ref{alg:hamu-u}.
\begin{algorithm}[t]
\caption{\reciprocal}
\label{alg:hamu-u}
\begin{algorithmic}[1]
\STATE {\bfseries Input:} initial weights $\vw_0$, retain data $D_r$, forget data $D_f$, batch size $m$, learning rate $\eta$, maximum iterations $T$, retain quality constraint $U>0$
\FOR{$t = 0,1,\dots,T-1$}
    \STATE Sample i.i.d.~size-$m$ batches $B_r \!\sim\! D_r$ and $B_f \!\sim\! D_f$
    \STATE $\grb \gets \nabla_{\vw} L_t(B_r); \ \gfb \gets \nabla_{\vw} L_t(B_f)$
    \STATE $R \gets \eta \norm{\gfb}$
    \IF{$U > R \norm{\grb}$}
        \STATE \textbf{break}
    \ENDIF
    \STATE $\bar{\kappa} \gets \grb \cdot \gfb$
    \STATE $\bar{\kappa}_3 \gets -U \norm{\gfb} / {R}$
    \STATE $\bar{\kappa}_4 \gets \sqrt{\lp \norm{\grb} \norm{\gfb} \rp^2 - {U^2 \norm{\gfb}^2} / {R^2}}$
    \IF{$\bar{\kappa} > \bar{\kappa}_4$}
        \STATE \textbf{break}
    \ENDIF
    \IF{$\bar{\kappa} \le \bar{\kappa}_3$}
        \STATE $\dw \gets \dfrac{R}{\norm{\gfb}} \cdot \gfb$
    \ELSE
        \STATE $\gfbp \gets \gfb - \dfrac{\bar{\kappa}}{\norm{\grb}^2} \cdot \grb$
        \STATE $\dw \gets - \dfrac{U}{\norm{\grb}^2} \cdot \grb + \sqrt{R^2 - \dfrac{U^2}{\norm{\grb}^2}} \cdot \dfrac{\gfbp}{\norm{\gfbp}}$
    \ENDIF
    \STATE $\vw_{t+1} \gets \vw_t + \dw$
\ENDFOR
\STATE \textbf{return} $\vw_t$
\end{algorithmic}
\end{algorithm}

\section{Adapting Stopping Criterion for Layer-Specific Constraints}
\label{app:stopping_criterion}
In Alg.~\ref{alg:hamu-q}, the stopping criterion is specified to be
\begin{equation*}
\gr \cdot \gf
= \kappa
> \kappa_2
= \sqrt{
    \lp \norm{\gr} \norm{\gf} \rp ^2 - \frac{Q^2}{R^2} \norm{\gr}^2
}.
\end{equation*}
This inequality is equivalent to
\begin{equation}
\label{ineq:q_stopping_criterion}
Q
> \frac{R}{\norm{\gr}} \sqrt{
    \lp \norm{\gr}\norm{\gf} \rp ^ 2
    - \lp \gr \cdot \gf \rp ^ 2
    }.
\end{equation}
As described in Sec.~\ref{sec:layer_specific_constraint}, the global constraint $Q$ can be distributed into layer-wise constraints $\lp Q_i \rp_{i=1}^{\ell}$, and the stopping criterion has to be adapted accordingly.
One possible way of terminating the unlearning process could be to replace all variable in Ineq.~\eqref{ineq:q_stopping_criterion} with their layer-wise counterparts,
\begin{equation}
\label{ineq:q_stopping_criterion_layer}
Q_i
> \frac{R}{\norm{\gr^{(i)}}} \sqrt{
    \lp \norm{\gr^{(i)}}\norm{\gf^{(i)}} \rp ^ 2
    - \lp \gr^{(i)} \cdot \gf^{(i)} \rp ^ 2
    },
\end{equation}
such that unlearning terminates when this layer-wise stopping criterion is met for any layer $i$ in the model.

However, as described in Sec.~\ref{sec:layer_specific_constraint}, different layers in the model can exhibit varying sensitivities to the same data.
It might be harder to unlearn at certain layers compared to other layers, such that the layer-wise stopping criterion in Ineq.~\eqref{ineq:q_stopping_criterion_layer} is triggered in an earlier unlearning epoch for layers that are harder to unlearn.
Terminating when any layer meets this criterion would result in an undesirably early termination of unlearning, when other layers have not yet met the criterion.

Instead, considering the layer-wise constraints are distributed with $\sum_{i=1}^{l} Q_i = Q$, we opt to terminate when this aggregate condition across all layers is met, 
\begin{equation}
\label{ineq:q_stopping_criterion_sum}
Q = \sum_{i=1}^{\ell} Q_i
> \sum_{i=1}^{\ell} \frac{R}{\norm{\gr^{(i)}}} \sqrt{
    \lp \norm{\gr^{(i)}}\norm{\gf^{(i)}} \rp ^ 2
    - \lp \gr^{(i)} \cdot \gf^{(i)} \rp ^ 2
    }.
\end{equation}
Using Ineq.~\eqref{ineq:q_stopping_criterion_sum} as the stopping criterion allows layers that are easier to unlearn to compensate for layers that are harder to unlearn, such that unlearning only terminates when the required forget quality threshold exceeds the value that can be sustained with the aggregate of all layers in the model.
This is the stopping criterion that we have used in Fig.~\ref{fig:stopping_criterion}.

\section{Unlearning Algorithms}
\label{app:baselines}
In this paper, we compare \ours against the following loss-based and gradient-based unlearning algorithms.
We use the same learning rate of \num{1e-4} for all unlearning algorithms in our experiments.

\paragraph{Gradient ascent~(GA)}
The gradient ascent unlearning algorithm~\citep{jang2023knowledge} only trains on the forget data, and uses a negated loss function to increase the loss on the forget data.

\paragraph{Fine-tune~(FT)}
The fine-tune unlearning algorithm trains the model using only the retain data, to decrease the loss on the retain data.

\paragraph{Gradient difference~(GDiff)}
The gradient difference unlearning algorithm~\citep{liu2022backdoor} combines GA and FT, with gradient ascent on the forget data and gradient descent on the retain data.

\paragraph{KL minimization~(KL)}
The KL minimization algorithm~\citep{maini2024tofu} combines gradient ascent on the forget data with minimizing the Kullback-Leibler divergence between the prediction of the model being unlearned and the original model before unlearning on the retain data.

\paragraph{SCRUB}
The SCRUB algorithm~\citep{Kurmanji2023-xm} uses a weighted combination of different losses, by maximizing the KL divergence between the predictions of the unlearned model and original model on the forget data, while minimizing the loss and KL divergence of the predictions on the retain data. 
Following the source code\footnote{\url{https://github.com/meghdadk/SCRUB}} of SCRUB, we set $\alpha = 0.001$ and $\gamma = 0.99$.

\paragraph{Gradient Rectified Unlearning~(GRU)}
The Gradient Rectified Unlearning~(GRU) algorithm~\citep{wang2025gru} projects the gradient on the forget data, $\gf$, such that the new vector is orthogonal to the gradient on the retain data, $\gr$, when the angle between the two gradients is obtuse.

\paragraph{Projecting Conflicting Gradients~(PCGrad)}
We adapt the projecting conflicting gradients algorithm~\citep{yu2020gradient} to our unlearning setting.
If the angle between $\gf$ and $\gr$ is obtuse, $\gf$ is projected to be orthogonal to $\gr$ and vice versa for $\gr$, and the mean of the two projected gradients is used for the weight update.

\paragraph{\ours}
For our unlearning algorithm \ours, we perform weight updates according to Alg.~\ref{alg:hamu-q} and Alg.~\ref{alg:hamu-u}.

\subsection{Implementation Details}
For unlearning algorithms that use both the retain and forget data~(\ours, GDiff, KL, SCRUB), we duplicate the forget data to match the size of the retain data.
For unlearning algorithms which only use the retain data~(FT) or the forget data~(GA), we duplicated smaller unlearning datasets to match the size of the dataset used when using both retain and forget data~(\ours, GDiff, KL, SCRUB).
For fair comparison, all algorithms were re-implemented with the same training code, with only the specific weight update step swapped out for each unlearning algorithm.
The same GPU optimizations described in Sec.~\ref{app:parallelized} were used for all methods which require gradients of both the forget and retain data~(GDiff, KL, SCRUB, GRU, PCGrad, \ours).
This provides an additional advantage for other unlearning methods, and might result in lower resource requirements than those reported in the original papers.
In our experiments, we ran all unlearning algorithms using the same learning rate for the same number of unlearning steps on the same hardware and number of GPUs.
Unless otherwise stated, we use $42$ as the random seed in all experiments.

\section{CV Experimental Setup}
\label{app:cv_exp_setup}
For training and unlearning of models for CV task, our code is adapted from the Huggingface Trainer\footnote{\url{https://huggingface.co/docs/transformers/en/main_classes/trainer}}.

\subsection{Dataset Construction}
\label{app:cv_forget_construction}
We used the CIFAR-10 dataset, which consists of $50000$ training samples, with $10$ classes of equal size.
No pre-processing transformations~(scale, flip or zoom) were used on the training dataset.

To construct the forget data, we first choose a random class $c$ which contains $n$ samples.
Given the parameter $\rho \in [0, 1]$, the forget data consists of $n - \floor{\rho n}$ randomly sampled data points from class $c$, and $\floor{\rho n}$ randomly sampled data points from the union of all other classes and data points in class $c$ that are not selected in the first step.
\begin{itemize}
    \item $\rho = 0$: forget one entire class, easy instance
    \item $\rho = 1$: forget i.i.d. samples from all classes, hard instance
\end{itemize}
In our experiments, we used $c = 0$ and $\rho \in \lc 0, 0.25, 0.5, 0.75, 1 \rc$.

\subsection{Training of CV Model}
\label{app:cv_training_setup}
We used the ResNet-20 architecture~\citep{he2016deep} with randomly initialized weights.
We used a fixed learning rate of \num{1e-3} with the AdamW optimizer~\citep{adamw}, batch size of $5000$ and trained for $50$ epochs.
Every 5 epochs, we performed evaluation on the CIFAR-10 test set, which consists of $10000$ samples across the $10$ classes, and saved the checkpoint model weights.
The final trained model was selected to be the checkpoint with the highest average retain and forget data evaluation accuracy.
Training was done on $1$ NVIDIA L40 GPU, requiring approximately $10$ minutes for each model.

\subsection{Unlearning of CV Model}
\label{app:cv_unlearning_setup}
To perform unlearning, we used a batch size of $5000$ samples per GPU, and trained for a maximum of $25$ epochs.
We used a fixed learning rate or \num{1e-4} for all unlearning algorithms, without any momentum-based optimizers.
Additional discussion and results with the use of optimizers are in App.~\ref{app:practical_considerations}.
Unlearning was done on $2$ NVIDIA L40 GPUs, with each run of the unlearning experiment requiring approximately $10$ minutes~(assuming no early termination).
The unlearning dataset is shuffled once before the start of unlearning, and is not shuffled between unlearning epochs.

\section{LLM Experimental Setup}
\label{app:llm_exp_setup}
For LLM training and unlearning, our code is adapted from the Supervised Fine-Tuning~(SFT) Trainer\footnote{\url{https://huggingface.co/docs/trl/main/en/sft_trainer}}.

\subsection{Dataset Construction}
\label{app:llm_forget_construction}
We used the WaterDrum-TOFU unlearning dataset for the LLM experiments, which consists of QA pairs for an LLM QA task.
This dataset extends the TOFU unlearning dataset~\citep{maini2024tofu} by constructing different variations of the QA pairs of different levels of similarity, watermarked with the text watermarking scheme Waterfall~\citep{lau2024waterfall}.
This dataset contains different splits, of which we used the watermarked ``No duplicate'' and ``Semantic duplicate'' splits for $10\%$ of the forget data.
We refer to the settings using these two splits as ``semantically dissimilar'' and ``semantically similar''.

Following the implementation in~\citet{lu2025waterdrum}, the ``semantically dissimilar'' setting consists of $400$ samples in the forget data, and $3600$ samples in the retain data, while the ``semantically similar'' setting consists of $400$ samples in the forget data, and $4000$ samples in the retain data, of which $400$ samples in the retain data are paraphrased versions of the $400$ samples in the forget data.

\subsection{Training of LLM}
\label{app:llm_training_setup}
The original full LLM was trained from the base LLM \texttt{Llama-2-7b-chat-hf}.
We used LoRA~\citep{hulora} with $r=8$ and $\alpha=16$, batch size of $64$ and trained for 10 epochs.
We used a learning rate of \num{1e-4} with the AdamW optimizer~\citep{adamw} during LLM training.
Training was done on $1$ NVIDIA H200 GPU, requiring approximately $20$ minutes for each LLM.

\subsection{Unlearning of LLM}
\label{app:llm_unlearning_setup}
We used a batch size of $50$ samples per GPU, and trained for a maximum of 720 steps~(equivalent of 10 epochs on $D_r$).
We performed unlearning without using any momentum-based optimizers.
Unlearning was done on $2$ NVIDIA H200 GPUs, with each run of the unlearning experiment requiring approximately $20$ minutes~(assuming no early termination).
The unlearning dataset is shuffled once before the start of unlearning, and is not shuffled between unlearning epochs.

For this experiment, we used the \texttt{WaterDrum-TOFU} dataset with $10\%$ of the training dataset being the forget data, where the retain data is $9$ times the size of the forget data.
When only using the retain data~(e.g.,~FT), we duplicated the retain data $2$ times.
When only using the forget data~(e.g.,~GA), we duplicated the forget data $18$ times.

In the experiments, we terminated \primal once the feasibility criterion $Q \leq R \norm{\gf}$ of Problem~\eqref{problem:opt-ldr} fails.
Additionally, to reduce the total time required for all experiments, we implemented some early termination checks in the unlearning process.
We terminate the algorithm when forget loss decreases below $0.1$~(model is not unlearning on the forget data), or retain loss increases above $0.4$~(model has unlearned too much on the retain data).
For reference, the losses on the forget and retain data are both approximately $0.3$ before unlearning.

\section{Practical Considerations with Modern Deep Learning Frameworks}
\label{app:practical_considerations}
\subsection{Parallelized Computation of Forget and Retain Gradients}
\label{app:parallelized}
As our proposed unlearning algorithm requires the use of and comparison between the gradients of the retain samples $\grb$ and gradients of the forget samples $\gfb$, both copies of the gradients are required to be available on the GPU at each iteration. 
A naive approach at each step would be to perform the forward and backward pass on the retain samples, store the gradients separately, then perform the forward and backward pass on the forget samples, and finally compare the two gradients.
This incurs a high memory cost on the GPU, as memory has to be allocated for both $\grb$ and $\gfb$.
As such, we would only be able to process half the number of samples at each iteration as compared to if we only needed one copy of the gradients based on the memory available on the GPU.
This is also true even if this na{\"i}ve approach is used with multi-GPU model parallelism, where each GPU still has to allocate memory for two copies of gradients.

To perform this efficiently, we implemented a parallelized unlearning script, which allocates the forget and retain samples to alternate GPUs when ran on an even number of GPUs.
This allows half of the GPUs to perform computation on the retain samples while the other half of the GPUs perform computation on the forget samples at the same time.
Each half of the GPUs would only require memory to store either $\grb$ or $\gfb$~(but not both), allowing us to make full use of compute and memory resource on all GPUs.
As each GPU only needs enough memory to store a single copy of its own gradients, we would effectively be able to process double the batch size with the same memory available, as compared to the na{\"i}ve implementation~(or enable \ours to be used when two copies of the gradients cannot fit within the memory of a single GPU).

To implement this, we made use of the integration between the SFT Trainer and \texttt{accelerate} library~\citep{accelerate} for a multi-GPU setup.
During the data preparation step, we interleave the forget and retain samples into the unlearning dataset according to the batch size of each GPU, such that the dataloader would send forget samples to one half of the GPUs and retain sample to the other GPUs.
During the backward pass, instead of simply aggregating the gradients from all GPUs as in a typical multi-GPU setup, we instead aggregate the gradients from alternate GPUs responsible for the forget and retain samples separately.
When comparison of $\grb$ and $\gfb$ are required, each layer's gradient is iteratively sent over to the other GPUs as needed, such that only a small amount of additional memory is required to store only the single layer's worth of gradients.
These gradients are then passed into the weight update function which computes the final weight update according to the unlearning algorithm.
Exact implementation is in our provided code.

Table~\ref{tab:compute_time} shows the GPU memory and computation cost of the different unlearning methods under the CV experimental setting~(App.~\ref{app:cv_exp_setup}).
All methods are evaluated under the same hardware and process the same number of samples.
Reported timing excludes computations other than unlearning, such as script startup, evaluation and logging operations.
Results show \ours only incurs modest computational overhead compared to standard baselines, and remains comparable to other multi-objective methods~(KL, SCRUB).
For memory, \ours does not incur additional GPU memory as it only fetches gradients from other GPU when needed and performs computation in-place, whereas KL and SCRUB require more GPU memory to store a copy of the original model.

\begin{table}[t]
    \centering
    \caption{Average computational time per epoch and peak GPU memory usage of unlearning across the different unlearning methods under the CV setting.}
    \begin{tabular}{c|ccccccc}
    \toprule
    Method & FT    & GA    & GDiff & KL    & SCRUB & \primal & \reciprocal \\
    \midrule
    Unlearning time (seconds/epoch) & 0.879 & 0.884 & 0.868 & 1.067 & 1.064 & 0.976  & 0.994  \\
    Peak GPU memory (GB) & 7.65 & 7.65 & 7.65 & 8.87 & 8.87 & 7.65 & 7.65 \\
    \bottomrule
    \end{tabular}
    \label{tab:compute_time}
\end{table}

\subsection{Using \ours with Optimizers}
\label{app:optimizer}
In the main paper, our experiments are set up in a way to clearly demonstrate the theoretical formulation in Problem~\eqref{problem:opt-ldr} and Problem~\eqref{problem:opt-ldf}.
However, when deploying unlearning algorithms in practice on real world ML models, there are often other practical constraints and requirements, as well as practical optimizations that could be made to improve the overall performance.
For practical constraints, the model owner might be limited to a small number of unlearning steps~(e.g.,~only $1$ epoch) such that the computational cost of unlearning remains small relative to retraining.
For practical optimizations, the model owner typically implements momentum-based optimizers~(e.g.,~AdamW~\citep{adamw}) to accelerate the convergence of the optimization process.

Optimizers such as AdamW~\citep{adamw} speed up the optimization process by incorporating the weight updates in previous iterations to the update in the current iteration.
When the gradients do not change significantly between iterations, the step size taken by each weight update gradually increases.
This speeds up the unlearning process, and the model is able to achieve a similar or better forget quality and retain utility with fewer number of unlearning steps.

To allow our weight update to work with momentum-based optimizers, instead of directly performing weight update with $\vw_{t+1} = \vw_t + \dw^*$, we instead compute the equivalent gradient which would have resulted in a gradient descent weight update, $\g{}^* = -\frac{\dw^*}{\eta}$.
This equivalent gradient $\g{}^*$ is then passed to the momentum-based optimizers.
When we set $R = \eta \norm{\gr}$, the equivalent gradient for \primal is 
\begin{equation} \label{eq:equivalent-grad-ldr}
    \g{}^* = 
    \begin{dcases}
        \gr \quad & \text{if } \kappa \leq \kappa_1 \\
        - \frac{Q \norm{\gr}}{R \norm{\gf}} \cdot \gf + \sqrt{\norm{\gr}^2 - \lp \frac{Q \norm{\gr}}{R \norm{\gf}} \rp^2} \cdot \frac{\perpen{\gr}}{\norm{\perpen{\gr}}} \quad & \text{otherwise}.
    \end{dcases}
\end{equation}
Fig.~\ref{fig:exp_optimizer} shows that \ours is compatible with the use of optimizers, and allows for unlearning to be performed with fewer number of unlearning steps.
Using AdamW with only $1$ epoch achieves a higher $\Delta L_f$ and $-\Delta L_r$ than vanilla SGD in $5$ epochs.
\begin{figure}
    \centering
    \includegraphics[width=0.49\linewidth]{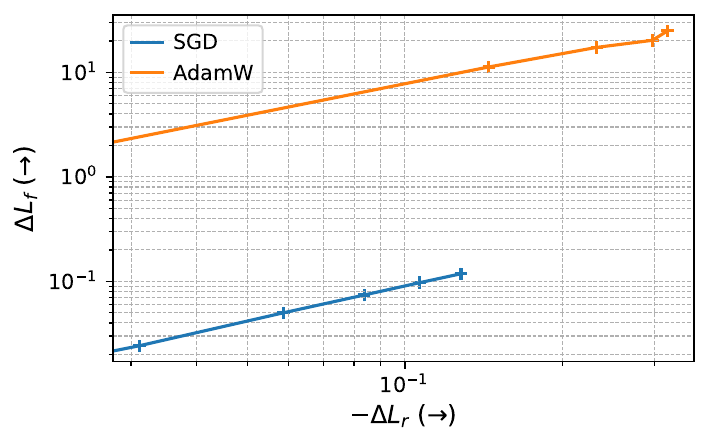}
    \caption{Comparing unlearning trajectory when using \reciprocal~($Q=0.9\eta$) with vanilla SGD versus AdamW, for 5 unlearning epochs.
    Note that the axes are plotted in log scale.
    A larger increase in $\Delta L_f$~(and $-\Delta L_r$) is achieved with the same number of unlearning epochs for AdamW than vanilla SGD.}
    \label{fig:exp_optimizer}
\end{figure}

\subsection{Sensitivity analysis for unlearning batch size}
\label{app:ablation_batch_size}
In our experimental setup, we used a large batch size that could still fit within the available memory in the GPU.
However, there will be situations where a smaller batch size has to be used, such as when a GPU with less memory is used, the model being unlearned is large with more parameters, or text training data for LLMs has longer sequence lengths.
A smaller batch size can result in several possible issues, such as poorer approximation and high variance between iterations of batch gradient estimates when updating the weights.

To resolve these issues, we have implemented several possible solutions in our code.
Firstly, our code can be run with more than 2 GPUs, such that alternate GPUs handle the gradient computation of the retain and forget data respectively.
When more GPUs are available, it results in a larger effective batch size.
Secondly, our code is also fully compatible with gradient accumulation, where gradient values are collated across multiple iterations of the unlearning loop before the weight update is performed, similarly achieving a larger effective batch size.
Lastly, to address the issue of fluctuations in $\bar{\kappa}$ prematurely triggering the stopping criterion, practical adjustments can be made to the stopping condition, such as implementing a moving average of $\kappa / \kappa_2$, or only stop when the stopping criterion is met for a certain number of times.

To evaluate the effect of different batch size on \ours, we conducted ablation experiments, varying the batch size used in the CV experiments.
Table~\ref{tab:ablation_batch_size} shows the unlearning performance of \primal with batch size between $128$ and $2048$.
We observe that smaller batch sizes tend to achieve better forget quality, likely due to more frequent weight updates for a fixed number of training samples~(a smaller batch size will have more weight update steps within a single epoch).
Retain utility is relatively robust to batch size, but very small batches~(128) degrade the retain utility, and increase runtime due to poor GPU utilization.
Under this experimental setting, batch size 512 provides a good balance.

\begin{table*}[ht]
\caption{Sensitivity analysis with varying batch size.}
\label{tab:ablation_batch_size}
\centering
\begin{tabular}{c|cccc}
\toprule
Batch size & $-\Delta L_r$↑ & $\Delta L_f$↑ & Retain accuracy (\%)↑ & Forget accuracy (\%)↓ \\
\midrule
128        & -0.416      & 14.7       & 72.7    & 0.0     \\
256        & 0.146       & 6.04       & 84.9    & 2.4     \\
512        & 0.246       & 2.65       & 88.8    & 18.1    \\
1024       & 0.252       & 1.11       & 88.9    & 43.2    \\
2048       & 0.201       & 0.48       & 86.2    & 58.8    \\
\bottomrule
\end{tabular}
\end{table*}

\section{Additional Experimental Results}
\label{app:exp_extra}
\subsection{Baseline Comparison for Varying Data Similarity Mixing Ratio for CV Task}
\label{app:additional_baseline_cv}
Fig.~\ref{fig:app:exp_cv_baseline} shows a more comprehensive plot of Fig.~\ref{fig:exp_cv_baseline} for the CV task.
We include the unlearning trajectories for more epochs, $25$ epochs\footnote{The computation cost of $25$ epochs of unlearning is similar to that required for retraining~(training was performed for $50$ epochs, but on one GPU as opposed to two GPUs for unlearning), which is unrealistic in practice, since a model owner could simply retrain the model from scratch rather than to perform unlearning.
Despite this, we show Fig.~\ref{fig:app:exp_cv_baseline} so as to clearly illustrate the differences between the different unlearning algorithms.
}, compared to 5 epochs in Fig.~\ref{fig:exp_cv_baseline}, for $\rho \in \lc 0.0, 0.25, 0.5, 0.75, 1.0 \rc$.
These results corroborate the conclusions in Fig.~\ref{fig:exp_cv_baseline}, where only \ours performs well in harder unlearning scenarios.
We can clearly observe that as the similarity mixing ratio increases, all unlearning methods perform progressively worse, with a smaller $\Delta L_f$ and $-\Delta L_r$.
Notably, the unlearning performance of the two better performing baselines, GDiff and SCRUB, notably deteriorates beyond $\rho=0.25$.
GDiff is only able to achieve very small magnitude of $\Delta L_f$ and $-\Delta L_r$, while for SCRUB, $\Delta L_f$ becomes negative.
Here, we additionally include the unlearning results for $\rho = 1.0$, where the forget data contains random samples of all classes from the overall training data, with sample distribution indistinguishable from that of the retain data.
It can also be noted that when the forget and retain data becomes too similar~($\rho = 1.0$), it becomes impossible to decouple $\Delta L_f$ from $\Delta L_r$, regardless of the unlearning algorithm used.

\begin{figure}[htp]
    \centering
    \includegraphics[width=1\linewidth]{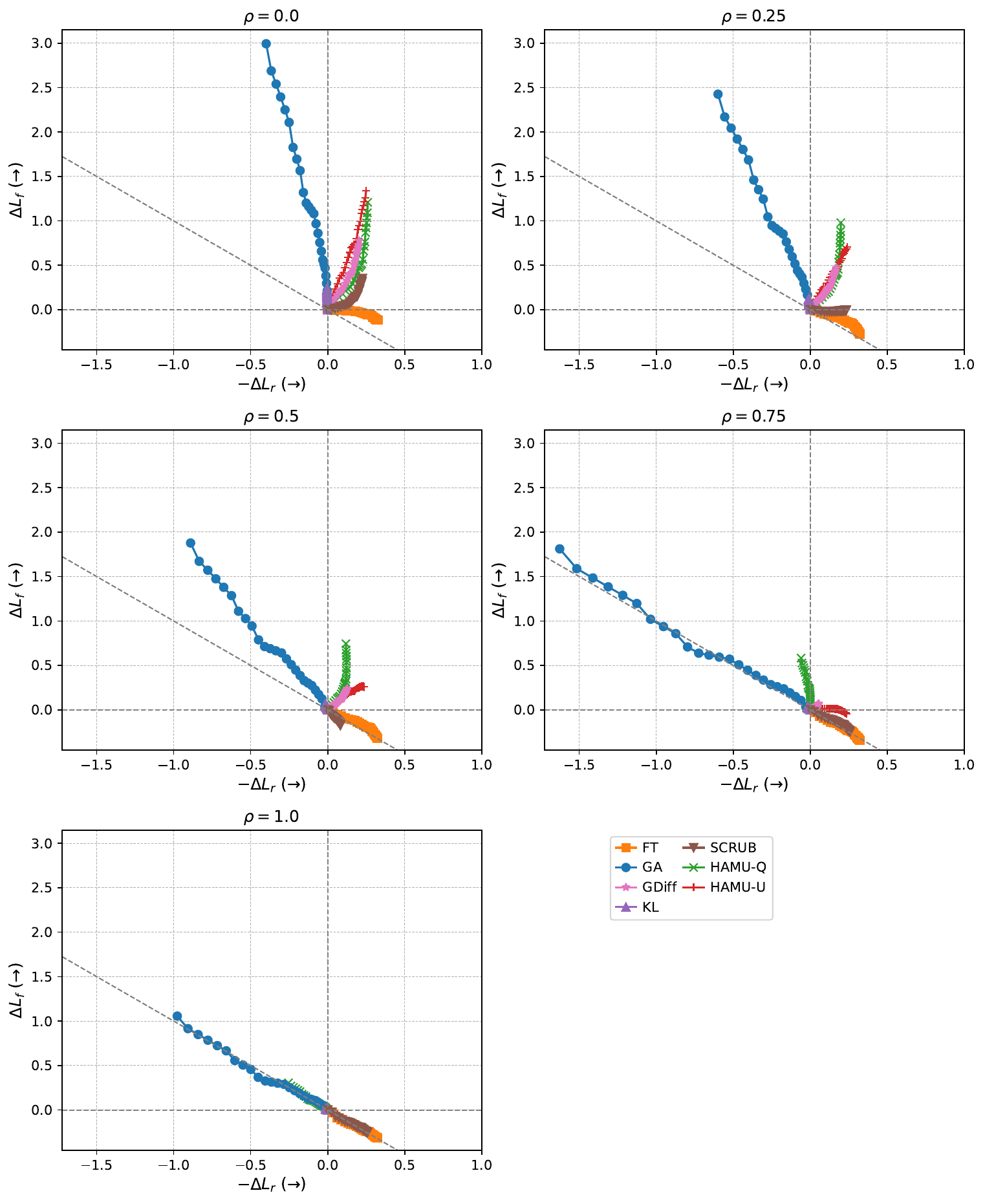}
    \caption{More comprehensive plot of Fig.~\ref{fig:exp_cv_baseline}, comparing \ours with baseline unlearning algorithms for different level of $\rho$ under the CV experimental setting.}
    \label{fig:app:exp_cv_baseline}
\end{figure}

\subsection{Additional Evaluation Metrics and Baselines for LLM QA Task}
\label{app:additional_baselines_metrics}
We provide results for additional evaluation metrics: token accuracy, model utility~(probability)~\citep{maini2024tofu}, and privacy~(MIA)~\citep{shi2024detecting} metrics.
We also include two additional baselines, Gradient Rectified Unlearning~(GRU)~\citep{wang2025gru} and Projecting Conflicting Gradients~(PCGrad)~\citep{yu2020gradient}.
PCGrad as described in~\citet{yu2020gradient} performs multi-task optimization and was not initially designed to be used in unlearning.
We adapt PCGrad's update rule to make use of the gradient of the retain data and the negated gradient of the forget data as the task gradients.
In this experiment, we focus on the more realistic setting of unlearning with optimizers~(described in App.~\ref{app:optimizer}), using learning rate of \num{1e-5}, and unlearn for 3 epochs.

Table~\ref{tab:additional_baselines_metrics} shows that the trend across methods in Retain Acc. and Model Utility follows that of $-\Delta L_r$, and Forget Acc. follows that of $\Delta L_f$.
Thus, \ours achieves a better trade-off and is the only method better than before unlearning on both retain utility and forget quality metrics.
\primal achieves a better MIA~(closer to $0.5$) as it guarantees sufficient improvement in forget quality unlike \reciprocal and GRU which focus on the retain utility instead.

\begin{table}[h]
    \centering
    \caption{Results for additional baselines and evaluation metrics.}
    \label{tab:additional_baselines_metrics}
    \begin{tabular}{crrcccc}
        \toprule
        Method & $-\Delta L_r \uparrow$ & $\Delta L_f \uparrow$ & Retain Acc.(\%) $\uparrow$ & Forget Acc.(\%) $\downarrow$ & Model Utility $\uparrow$ & MIA \\
        \midrule
        Before unlearning & 0.000 & 0.000 & 0.924 & 0.907 & 0.779 & 0.999 \\
        FT               & 0.216   & -0.104 & 0.997 & 0.941 & 0.966 & 1.000 \\
        GA               & -105.392 & 105.552 & 0.002 & 0.002 & 0.000 & 0.447 \\
        GDiff            & -1.452 & 11.284 & 0.769 & 0.485 & 0.188 & 0.186 \\
        KL               & -0.867 & 7.061 & 0.805 & 0.593 & 0.335 & 0.098 \\
        SCRUB            & -0.290 & 13.179 & 0.878 & 0.562 & 0.575 & 0.216 \\
        GRU              & -105.213 & 105.326 & 0.002 & 0.002 & 0.000 & 0.457 \\
        PCGrad           & -1.443 & 11.325 & 0.769 & 0.485 & 0.189 & 0.187 \\
        \primal          & 0.151 & 2.075 & 0.976 & 0.754 & 0.907 & 0.552 \\
        \reciprocal      & 0.008 & 4.372 & 0.931 & 0.663 & 0.788 & 0.198 \\
        \bottomrule
    \end{tabular}
\end{table}

\end{document}